\documentclass{article}

\PassOptionsToPackage{numbers, sort&compress}{natbib}

\usepackage[dblblindworkshop, final]{neurips_2026}
\workshoptitle{Foundations of Efficient Deep Learning}

\usepackage[utf8]{inputenc} 
\usepackage[T1]{fontenc}    
\usepackage{float}          
\usepackage{hyperref}       
\usepackage{url}            
\usepackage{booktabs}       
\usepackage{amsfonts}       
\usepackage{nicefrac}       
\usepackage{microtype}      
\usepackage{xcolor}         

\usepackage{amsmath,amsfonts,bm}

\def\eqref#1{equation~\ref{#1}}

\def\1{\bm{1}}

\DeclareMathAlphabet{\mathsfit}{\encodingdefault}{\sfdefault}{m}{sl}
\SetMathAlphabet{\mathsfit}{bold}{\encodingdefault}{\sfdefault}{bx}{n}

\newcommand{\E}{\mathbb{E}}

\newcommand{\R}{\mathbb{R}}

\usepackage{amsthm, amsmath, graphicx}
\newcommand{\fitwidth}[1]{\resizebox{\ifdim\width>\linewidth\linewidth\else\width\fi}{!}{#1}}
\theoremstyle{definition}
\newtheorem{definition}{Definition}[section]
\newtheorem{theorem}{Theorem}[section]
\newtheorem{proposition}{Proposition}[section]

\newtheorem{lemma}[theorem]{Lemma}
\newtheorem{example}[theorem]{Example}

\newcommand{\norm}[1]{\left\lVert #1 \right\rVert}
\newcommand{\abs}[1]{\left\lvert #1 \right\rvert}

\newcommand{\es}{\mathrm{es}}
\newcommand{\EF}{\mathrm{EF}}
\newcommand{\ES}{\mathrm{ES}}
\newcommand{\mass}{A^{\star}}

\title{Trust the Mass: Forced Weights in KV-Cache Eviction}
\hypersetup{pdftitle={Trust the Mass: Forced Weights in KV-Cache Eviction},pdfauthor={Jack Shi, Jerry Gu}}

\author{Jack Shi\thanks{Equal contribution.} \\    
Department of Computer Science \\ Stanford         
University \\\texttt{jackshi@stanford.edu} \And    
Jerry Gu\footnotemark[1] \\ Department of 
Computer Science \\ Stanford University \\
\texttt{jerrygu@stanford.edu}}

\begin{document}

\maketitle
   \begin{abstract}
    Every deployed KV-cache-eviction rule and every top-$k$ sparse-attention rule keeps a subset of the keys, discards the rest, and renormalizes the attention weights over the kept set.  
    Identifying the exact best subset under that constraint, we find that keeping the largest weights is already near-optimal, since the best subset we find closes only a median $2$ to $5\%$ of the remaining gap to full attention. We evaluate deployed eviction methods in KVPress and find that methods with non-uniform per-head budgets apply attention masks over the full cache because dense cache tensors require equal length across heads. Truncating each head to the target budget lowers benchmark scores by $14$ to $62$ points, and computing token importance with the question in context accounts for $87.6$ points in retrieval accuracy. We propose ContourKV, a training-free rule that applies KVzip's score under a physically enforced budget. It beats KVzip in $93$ of $160$ benchmark runs across RULER and LongBench while using $28$ to $1078$ times less memory, and ties Compactor, the strongest baseline at equal memory.
    \end{abstract}

\section{Introduction}\label{sec:intro}
\looseness=-1 In modern transformers, the KV cache stores a key and a value vector for every past token at every attention head, which dominates memory at long context. Consequently, systems in production reduce memory consumption  by evicting entries. Common eviction rules (and top-$k$ sparse-attention rules) keep a subset of the keys, renormalizes the softmax weights over it, and returns the average in place of the dense attention output. Thus, eviction methods differ only in how the subset is chosen~\citep{snapkv,h2o,kvzip,compactor}. In contrast, classical coreset and subset-selection methods choose points geometrically, optimizing new weights over the retained set~\citep{varopt,design,feldman-langberg,balancekv}, while sampling-based estimators reweight keys drawn from the full cache~\citep{magicpig,vattention}. Eviction instead \textit{forces} the weights to be the renormalized softmax probabilities, leaving only the choice of subset. We quantify the maximum gain any kept subset can achieve under forced weights. Further, we identify what is actually measured in comparisons between eviction methods.

\looseness=-1 To do this, we enumerated the exact best subset under forced weights within a candidate set of keys on $220{,}656$ attention rows from eight models at each of four budgets, and we found that keeping the largest weights is already near-optimal. The best subset we find improves on it by a median $2$ to $5\%$ of the gap to the dense output. We found a cheap swap rule recovers that median in full.  Additionally, searching over all keys with a solver bounds the median improvement across the four main model families below $12\%$. The total attention weight of the evicted keys predicts the instances where larger gains occur (\S\ref{sec:perhead}).

In practice, deployed eviction rules incur much larger approximation errors than keeping the top-$k$ weights, so their reported gains reflect how well they predict future weights as well as differences in memory usage, query visibility, and compute. Henceforth, we use the standard query-agnostic evaluation pipeline, \texttt{kvpress}, to compare    
memory consumption across eviction methods. Several query-agnostic methods (such as KVzip) allocate different budgets per attention head and store their selection as a mask. Therefore, these methods actually measure selection quality at full memory (\S\ref{sec:bytes}) as opposed to accuracy under a physically compressed cache. Actually forcing each head to compress to the target budget (rather than letting heads keep extra tokens) lowers scores by $14.1$ to $62.2$ benchmark points (\S\ref{sec:wrapped}), and compressing the cache in a question-agnostic setting lowers retrieval accuracy by $87.6$ points (\S\ref{sec:needle}).

\looseness=-1 To test how KVzip's scoring performs 
under actual memory compression, we created        
ContourKV, a training-free eviction rule that ranks  
tokens by KVzip's importance score~\citep{kvzip} and
physically evicts them to a fixed per-head budget.   
Because KVzip masks the full cache in
\texttt{kvpress}, ContourKV uses $28$ to $1078$ times
less memory (App.~\ref{app:campaign}), yet it still  
beats KVzip in $93$ of $160$ benchmark runs across RULER and LongBench. Compared with methods that also physically shrink the cache, ContourKV ties Compactor~\citep{compactor}, the strongest baseline at equal memory.

\looseness=-1 Our primary contribution is measuring the limits of token selection and accounting for actual memory usage in eviction benchmarks, rather than developing a new eviction method (App.~\ref{app:related}).  For non-uniform per-head eviction algorithms like KVzip to actually save memory in multi-turn serving, where one compressed cache serves all future queries, they require storage that allows different cache lengths per head~\citep{tangram}.

\section{Operator and Frozen-Head Measurements}\label{sec:operator}
\label{sec:perhead}
\begin{definition}[Terminology]\label{def:instance}
An \textit{instance} is a strictly positive probability vector $p$ on $[N]=\{1,\dots,N\}$ with points $v_1,\dots,v_N\in \R^d$ and mean $\mu=\sum_j p_jv_j$. A nonempty \textit{kept set} $A\subseteq[N]$ has $p(A)=\sum_{j\in A} p_j$, where $p_j$ is the \textit{mass} of key $j$, and it has dropped mass $\bar p(A)=1-p(A)$ and $m_A=\sum_{j\in A}(p_j/p(A))v_j$. The top-mass set $A_s^\star$ is a size-$s$ set of largest masses and $D=\max_j\Vert v_j-\mu \Vert$.
\end{definition}

At an attention head, $p$ is one query row, $v_j$ are the values, $\mu$ is the dense output, and $m_A$ is what any deployed sparse operator returns. We call $\bar m_s=\bar p(A_s^\star)$ the least achievable dropped mass. For $|A|=s$, let $\mathrm{ef}(A)=\mathrm{dist}(\mu, \mathrm{conv}\{v_j:j\in A\})$ be the error under re-solved weights and $\mathrm{es}(A)=\Vert\mu-m_A\Vert$ the error under forced weights. Their minima over $|A|=s$ are $\mathrm{EF}(s)$ and $\mathrm{ES}(s)$, and $\kappa(s)=\mathrm{es}(A_s^\star)/\mathrm{ES}(s)\geq 1$.
\begin{lemma}\label{lem:identities}
For $0<p(A)<1$, $\mathrm{es}(A)=\bar p(A)\,\Vert m_{A^c}-m_A\Vert=\Vert g_A\Vert/p(A)$, $g_A:=\sum_{j\in A}p_j(v_j-\mu)$, and $p_j(v_j-\mu)$ sums to $0$ across all $j$. For every $A$, $\mathrm{es}(A)\geq \mathrm{ef}(A)$, so $\mathrm{ES}(s)\geq \mathrm{EF}(s)$, and $\mathrm{es}(A_s^\star)\leq 2D\bar m_s$.
\end{lemma}
The factorization and the last bound are known~\citep{tvtopk,ssa,caote,criticalkv}, although no prior bound addresses what we study, the gap to the best subset. The problem is solved for free weights, where $\mathrm{EF}(s)\leq D/\sqrt s$ and the minimizers are geometric and do not depend on $p$ (App.~\ref{app:proofs}). Free and forced weights differ in value and argmin (Ex.~\ref{ex:bite}). These differences increases with budget (App.~\ref{app:penalty}).
\begin{proposition}[Computational Hardness]\label{prop:hard}
    Deciding $\mathrm{ES}(s)=0$ is \textsf{NP}-complete, even at $d=1$ with uniform masses, integer points and $s=N/2$. Unless \textsf{P}$=$\textsf{NP}, no polynomial-time estimate of $\mathrm{ES}(s)$ within any multiplicative factor exists, and none within additive error $<N^{-2}$ on the reduction class (App.~\ref{app:proofs}).
\end{proposition}
As the objective is not submodular (Ex.~\ref{ex:notsub}), and $\mathrm{ES}(s)$ is a nonconvex combinatorial minimization, it is natural to enumerate over a restricted candidate set and use a cheap subset with error upper-bounding $\mathrm{ES}(s)$ per budget. A selector\label{def:bal} maps an instance and a budget to a size-$s$ kept set (top-mass is the deployed one) and the \textit{balancing selector} is the greedy alternative that starts at $A=A_s^\star$ and repeatedly takes the swap of one kept key for one dropped key that most decreases $\mathrm{es}(A)$, stopping when no decrease exists. Balancing reads the dense output $\mu$, so it runs at decode time (where the value-aware criteria of~\citep{caote,vatp} also apply) or as a prefill teacher. We use it only to measure how much of the gap a better kept set can still close. In particular, ContourKV does not use it (\S\ref{sec:campaign}).

\begin{definition}\label{def:pool}
\looseness=-1 The \textit{candidate set} of an instance is the top-mass set together with the keys of largest single-key error $p_j\Vert v_j-\mu\Vert$, and $\mathrm{ES}_{\mathrm{pool}}(s)$ is the exact minimum of $\mathrm{es}$ over the size-$s$ subsets of the candidate set, so $\mathrm{ES}_{\mathrm{pool}}\geq \mathrm{ES}(s)$ and the derived estimates are lower bounds. Write $\hat\kappa=\mathrm{es}(A_s^\star)/\mathrm{ES}_{\mathrm{pool}}$, and write $c=1-\mathrm{ES}(s)/\mathrm{es}(A_s^\star)$ for the fraction of the gap that a subset can close, and $\hat c$ for its estimate from $\mathrm{ES}_{\mathrm{pool}}$. Let $A_{\mathrm{bal}}$ be the set that the balancing selector returns. A (layer, head) pair, called a \textit{cell}, is \textit{flat} when its median dropped mass at $s=8$ exceeds $0.25$.
\end{definition}

\begin{figure}[tbp]
\begin{minipage}[c]{0.56\textwidth}\includegraphics[width=\linewidth]{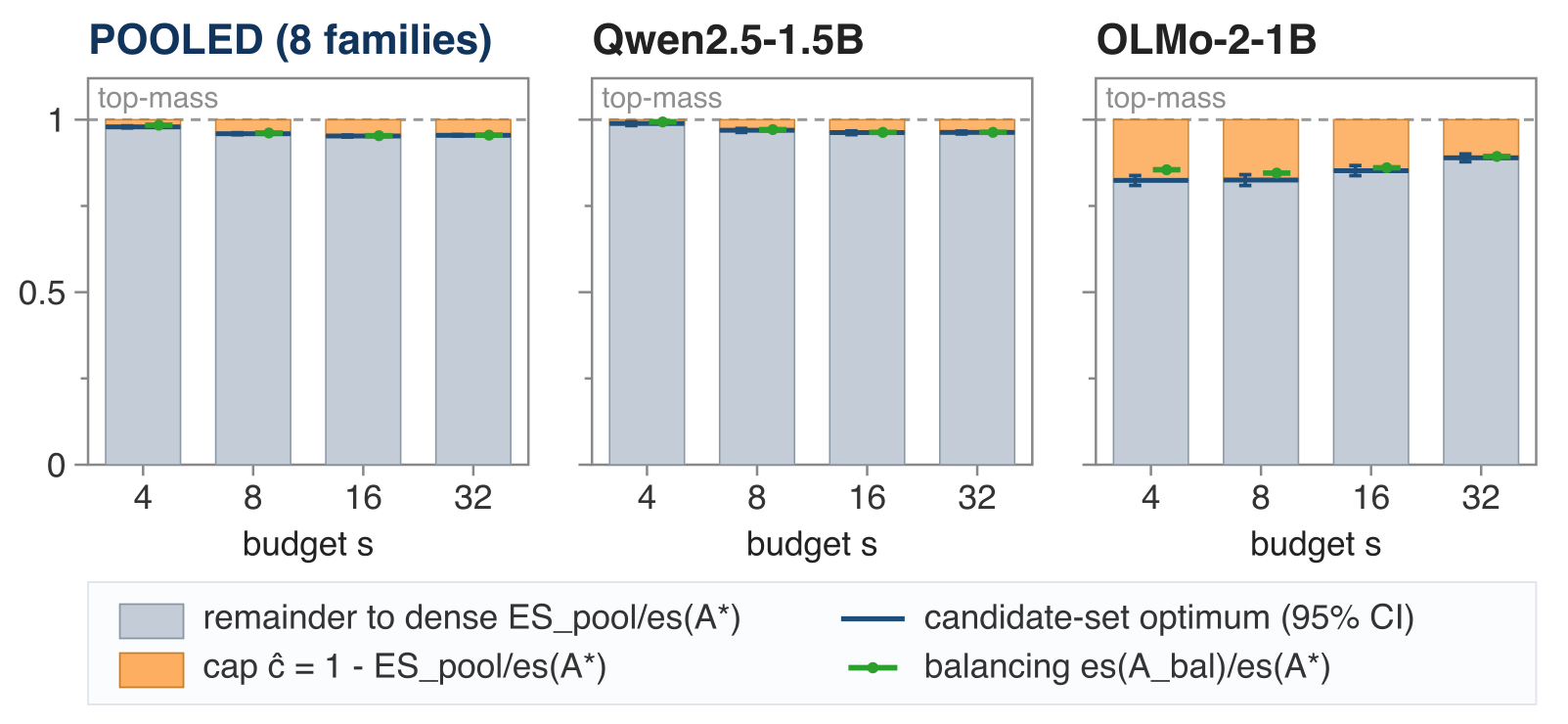}\end{minipage}\hfill
\begin{minipage}[c]{0.42\textwidth}\caption{\textbf{Ceiling.} The residual is normalized so that top-mass sits at $1$ and dense at $0$ (ten arms in Figure~\ref{fig:ceilingfull}). The balancing selector's median residual is within $0.031$ of the floor.}
\label{fig:ceiling}\end{minipage}
\end{figure}

\looseness=-1 Here, an instance is a sampled query row of a frozen attention head. Four models, Qwen2.5-1.5B (our reference model), Llama-3.2-1B, OLMo-2-1B, and Qwen3-1.7B, give $111{,}744$ instances at $s\in\{4,8,16,32\}$ from six rows sampled per head at every layer over three documents. We also ran Qwen2.5-7B at $16{,}384$ tokens ($56{,}448$ more instances) and five other models under the same protocol, for ten models from $0.6$B to $14$B. Our main medians pool twenty documents on each of the four models with four of the others, $220{,}656$ rows per budget. We compute $\mathrm{ES}_{\mathrm{pool}}$ by exhaustive enumeration and test it by expanding the candidate set and running an exact full-key solver (App.~\ref{app:protocol}).

\paragraph{Selection ceiling.}\label{sec:ceiling} \looseness=-1 We find a median $\hat c$ of $0.021$ to $0.047$ over eight families and all budgets, rising to $0.200$ to $0.254$ on the instances with $\hat\kappa>1.11$, and the per-family median runs from $0.000$ (Qwen3-0.6B) to $0.176$ (OLMo-2-1B). The share is measured against the top-mass-to-dense gap $\es(\mass_s)$, since we recover $\mu$ when every key is kept, and it is small because large gaps are uncommon (App.~\ref{app:protocol}). This share lower-bounds $c$, and the solver caps the median $c$ of the four headline families below $0.12$ at every budget (App.~\ref{app:cert}). The balancing selector reaches the bound and recovers a median $1.00$ of the closable gap in every family and budget ($[1.00,1.00]$, Figure~\ref{fig:ceiling}). We cannot search exactly at the deployed $s=64$ to $256$ (Prop.~\ref{prop:hard}), but no selector gains more than $\es(\mass_s)\le2D\bar m_s$ (Lemma~\ref{lem:identities}), and this falls to a median $28$ to $34\%$ of its $s=8$ value at $s=64$ and $8$ to $12\%$ at $s=256$. The balancing subset closes $0.04$ to $0.13$ of the gap across $s=64$ to $256$ on the two audited families, at or above its $s=32$ level, so this lower-bounds $c$ (Table~\ref{app:tab:anchor}, App.~\ref{app:cert}). Thus $\es(\mass_s)$ decreases with $s$ while the closable fraction stays level.

\paragraph{Dropped mass.}\label{sec:certificate} \looseness=-1 We find that $\bar m_s$ best predicts whether an instance has $\es(\mass_s)>0$ and $\hat\kappa>1.11$, at AUC $0.76$ to $0.85$ on the four headline families and $0.76$ to $0.89$ across the ten arms. The published absolute bounds rank $\es(\mass_s)$ at Spearman up to $0.98$ but fall to $0.21$ to $0.70$ against $\hat\kappa$~\citep{tvtopk}. The prediction replicates across content types and retrieval contexts (Table~\ref{app:tab:content}, App.~\ref{app:perheadtables}). The mechanism follows from \S\ref{sec:operator}, since a row with small $\bar m_s$ bounds the selector's gain at $2D\bar m_s$ while a row with large $\bar m_s$ leaves terms $p_j(v_j-\mu)$ that can be cancelled by a swap. We deploy $\bar m_s$, since, per cell, $\bar m_s$ gives the map of Figure~\ref{fig:map}, and flat-cell location is architectural, set by query-key normalization and windowing (Table~\ref{tab:certificate}).

\paragraph{Deployed selectors and loss.}\label{sec:audit}\label{sec:teeth} \looseness=-1 The deployed selectors SnapKV, H2O, Quest and StreamingLLM sit well above the optimum at small budgets, at a median $1.4$ to $2.0$ times $\mathrm{ES}_{\mathrm{pool}}$ per query at $s=8$ against $1.02$ to $1.03$ for top-mass on three of the four models, while TOVA's kept set is top-mass itself~\citep{snapkv,h2o,tova,quest,streamingllm}. SnapKV's error holds at $1.8$ to $2.2$ times the same-budget top-mass error through the deployed budgets; $\bar m_s$ ranks each mass-based rule's excess at Spearman $0.70$ to $0.84$ without query-key normalization. When one evicted cache serves every query, TOVA, H2O and SnapKV sit at $1.4$ to $4.2$ times $\ES_{\mathrm{pool}}$, and one kept set shared by a query group costs each row a median $1.14$ to $1.38$ times its own $\ES_{\mathrm{pool}}$ (App.~\ref{app:perheadtables}). Thus, almost all of a deployed rule's excess error is the distance from its kept set to each row's top-mass set, so margins between rules reflect how well they predict those sets. In loss, per-query SnapKV, H2O and StreamingLLM sit $1.15$ to $2.48$ nats above top-mass at $s=8$, against a top-mass-to-dense gap of $0.11$ to $0.33$ (App.~\ref{app:audittables}).
\section{Deployment}
\label{sec:campaign}
\paragraph{ContourKV.}\label{sec:ContourKV} \looseness=-1 In ContourKV, each attention head keeps a recency window of the last $w=\min(32,s/2)$ tokens plus the highest-scoring tokens under an importance score $r$. With the budget enforced, every head retains its recency window and its own top $s-w$ tokens, physically evicting the rest so that every head holds exactly $s$ entries. Without enforcement, all layers and heads share a single threshold $\tau$ chosen so that the cache averages $s$ entries per head across the model, allowing individual heads to retain more than $s$ tokens. For retrieval, a per-layer threshold assigns one cache length to all heads within a layer while letting that length vary across layers (App.~\ref{app:campaign}). We instantiate $r$ with KVzip's context-reconstruction score~\citep{kvzip} and with SnapKV's observation-window score~\citep{snapkv}, which requires no second pass. Whereas KVzip uses a global threshold, four attention sinks, and an attention mask over the full cache, enforced ContourKV uses a uniform per-head budget, a recency window, and physical eviction.

\paragraph{Setup and measured memory.}\label{sec:bytes} \looseness=-1 We evaluate Llama-3.1-8B-Instruct, Qwen2.5-7B-Instruct, Qwen3-8B, and Qwen3-14B on RULER~\citep{ruler} and LongBench~\citep{longbench} at $4{,}096$ to $32{,}768$ tokens. Using NVIDIA's \texttt{kvpress} pipeline~\citep{kvpress}, we compare $14$ library baselines at the same per-head budget $s$: nine that keep a uniform budget per head and five that allocate different budgets across heads, plus H2O in our harness (App.~\ref{app:campaign}). Because \texttt{kvpress} stores each layer's cache as a single tensor with equal slots per head, the five non-uniform per-head methods, including KVzip, apply their selection as an attention mask and retain $100\%$ of the full cache across every benchmark run (freeing this memory requires storage that supports different cache lengths per head,~\cite{tangram}). In contrast, the uniform-budget methods physically evict tokens; SnapKV, Compactor, and budget-enforced ContourKV all shrink the cache to $0.09$ to $3.6\%$ of its full size. Without budget enforcement, ContourKV evicts physically by padding each layer to its widest head, retaining $3.5$ to $84\%$ of the cache.

\paragraph{Benchmark results.}\label{sec:wrapped} \looseness=-1 Without budget enforcement and with the context-reconstruction score, ContourKV beats KVzip in $146$ of $157$ benchmark runs and loses $2$, with LongBench margins of $18.1$ to $30.3$ points (Table~\ref{app:tab:wtl}). Enforcing the per-head budget $s$ on that same selection, with the token ranking unchanged, lowers scores in $140$ of $157$ runs. The drop averages $14.1$ points on LongBench $s=32$ and reaches $62.2$ points on RULER $s=128$, showing that the unenforced gains come largely from retaining extra tokens beyond the budget $s$. Even with the budget enforced, ContourKV beats KVzip in $93$ of $160$ runs and loses $22$, while using the same memory as the uniform-budget baselines; $16$ of those losses occur at RULER $s=32$, where KVzip's full cache leads every uniform-budget method. Against SnapKV, enforced ContourKV gains $+10.8$ points across all $16$ runs at RULER $s=128$. This advantage comes from the reconstruction score, which takes $8.7$ to $13.6$ times the prefill time of dense attention (App.~\ref{app:campaign}), as ContourKV with the observation-window score ties SnapKV at equal memory. Enforced ContourKV also ties Compactor across the $160$ runs (Table~\ref{app:tab:wtl}) and outperforms CompressKV~\citep{compresskv2}, which trails by $9.39$ points on average across all $8$ RULER $s=128$ runs on the two models it supports (App.~\ref{app:campaign}). Finally, reallocating the same total memory across layers using the observation-window score\label{sec:square} produces no gain over SnapKV on these benchmarks ($0.13$ points lower across $154$ runs; App.~\ref{app:protocol}).

\paragraph{Retrieval.}\label{sec:needle} \looseness=-1 On needle retrieval, this per-layer version of ContourKV outperforms SnapKV from $s=64$ onward, reaching $0.37$ to $1.00$ accuracy against $0.00$ to $0.11$ at $s=128$ on three of four models (App.~\ref{app:campaign}). However, this margin does not come from cross-layer budget allocation: in the retrieval harness, eviction runs after the full prompt enters the cache, so the observation window scores tokens using the question rows. Across eight matched Llama-3.1-8B runs, evicting at the end of the document before the question arrives lowers retrieval accuracy by $87.6$ points, showing that query visibility accounts for nearly the entire gain (Table~\ref{app:tab:condition}). Because the RULER and LongBench benchmarks in \S\ref{sec:wrapped} compress the document before the question is appended, their results are unaffected. Together with the low selection ceiling in \S\ref{sec:perhead}, these confounds show that fair comparisons between eviction methods require physically enforced budgets and measured memory.

\paragraph{Limitations.}\label{sec:limits} \looseness=-1 In our analysis (\S\ref{sec:perhead}), we measure single-head attention approximation error, which does not guarantee optimality in next-token cross-entropy loss (App.~\ref{app:scope}). In addition, keeping the largest attention weights is near-optimal only for a given query, whereas KV-cache eviction must choose which tokens to keep before seeing future queries. Finally, our experiments only cover dense-attention models up to $14$B parameters and context lengths up to $16{,}384$ tokens (App.~\ref{app:limits}).


\newpage
\bibliography{main}
\bibliographystyle{plainnat}


\appendix
\raggedbottom

\section{Proofs}\label{app:proofs}
This appendix proves Lemma~\ref{lem:identities} and Proposition~\ref{prop:hard} from \S\ref{sec:operator}.
\subsection{Identities}\label{apx:sec:idts}
\begin{proof}[Proof of Lemma~\ref{lem:identities}]
Use the law of total expectation to write $\mu=p(A)m_A+\bar p(A)m_{A^c}$, so that $\mu-m_A=\bar p(A)(m_{A^c}-m_A)$. Taking norms yields the first identity. For the second form, we have \[m_A-\mu=(p(A))^{-1}\sum_{j\in A}p_j(v_j-\mu)=\frac{g_A}{p(A)}.\] The terms $p_j(v_j-\mu)$ sum to $0$ because their sum expands to $\mu-\mu$. Since $m_A$ is a convex combination of the kept points, \[\Vert \mu-m_A\Vert\geq \mathrm{dist}(\mu,\mathrm{conv}\{v_j:j\in A\}),\] so $\mathrm{es}(A)\geq \mathrm{ef}(A)$, and minimizing over $|A|=s$ preserves the inequality. For the bound $\es(\mass_s)\leq 2D\bar m_s$, note that both $m_A$ and $m_{A^c}$ are convex combinations of $v_j$ and lie in $\bar B(\mu, D)$, so their distance is at most $2D$. Applying the first form at $\mass_s$ completes the bound.
\end{proof}
\begin{proof}[Remark]
    The single-key case recovers the per-decision eviction error of \citet{caote,criticalkv}, since a single dropped key $j$ has $m_{A^c}=v_j$, and the first form then gives $\es(A)=p_j\Vert v_j-m_A\Vert$.
\end{proof}
\begin{example}[Free and forced weights]\label{ex:bite}
Let $d=1$, $v=(-1,1,5)$, $p=(0.45,0.45,0.10)$, $s=2$. Then $\EF(2)=0$ and $\ES(2)=\tfrac9{22}$, attained by $\{1,3\}$. The top-mass set $\{1,2\}$ gives $\tfrac12$, a factor $\tfrac{11}9$ above the optimum.
\end{example}
\begin{proposition}\label{apx:prop:benchmark}
    Working in $\ell_2$, we have \[\EF(s)\leq\sigma/\sqrt s\leq D/\sqrt s,\, \, \sigma^2=\sum_j p_j\Vert v_j-\mu\Vert^2.\]
    The inequality is tight in the worst case, and $\EF(d+1)=0$. The minimizers are extreme points of the hull that do not depend on the masses.
\end{proposition}
\begin{proof}
By exact Carath\'eodory~\citep{rockafellar}, $\mu=\E[X]$ for a random $X$ supported on the points. Averaging $s$ independent copies gives $\E\norm{\mu-\bar X_s}^2=\sigma^2/s\le D^2/s$, so at least one of the $s$ copies has an average within $\sigma/\sqrt s$ of $\mu$, using at most $s$ distinct points. Padding the set to exactly $s$ points only expands the hull, so $\EF(s)\le\sigma/\sqrt s$. Tightness is classical~\citep{caratheodory-approx,vershynin}. Exactness at $s=d+1$ is Carath\'eodory.
\end{proof}
\begin{example}[Small $\bar m_s$ does not bound $\kappa$]\label{ex:smallmass}
A small dropped mass does not make the top-mass set near-optimal on a given instance, because $\kappa=\es(\mass_s)/\ES(s)$ is a ratio of two errors that can both be small, and the denominator can shrink faster. A dropped-mass prediction of $\kappa$ can
therefore only be statistical, which is what \S\ref{sec:certificate} measures.
\end{example}

\begin{example}[Not submodular]\label{ex:notsub}
Let $d=1$, $v=(0,1,2,3)$, and $p=(\tfrac14,\tfrac14,\tfrac14,\tfrac14)$, so $\mu=\tfrac32$, and write kept sets by their points. Adding the key at $0$ to $\{2\}$ leaves $\es$ at $\tfrac12$ while adding it to the superset $\{1,2\}$ raises $\es$ from $0$ to $\tfrac12$, so the increment grows on the superset and $\es$ is not submodular. Adding the same key to $\{1\}$ raises $\es$ from $\tfrac12$ to $1$ while adding it to the superset $\{1,3\}$ lowers $\es$ from $\tfrac12$ to $\tfrac16$, so the increment shrinks and $\es$ is not supermodular.
\end{example}

\subsection{Hardness}\label{apx:sec:hard}
Our argument follows by a reduction from the \textsc{Partition} problem, which is \textsf{NP}-complete~\citep{gareyjohnson}: given nonnegative integers $u_1,\dots,u_n$, decide whether some subset $S\subseteq[n]$ satisfies $\sum_{i\in S}u_i=\sum_{i\notin S}u_i$. We reduce in two steps. We first pass to the balanced variant, \textsc{BalancedPartition}: given an even $m$ and $w\in\mathbb Z_{\ge0}^{\,m}$, decide whether some $A$ with $\abs A=m/2$ splits the sum equally. We then encode \textsc{BalancedPartition} into \textsc{ZeroES}: given a rational instance of our problem, decide whether $\ES(s)=0$. Instances are rational throughout, with $p_j\in\mathbb Q\cap(0,1)$ and $v_j\in\mathbb Q^d$ in binary encoding, and by Lemma~\ref{lem:identities}, $\ES(s)=0$ holds exactly when some size-$s$ subset has $g_A=0$. We also use one elementary fact: over the rationals, 
\[
\sum_{i\in Y}x_i=\sum_{i\in X\setminus Y}x_i
\quad\Longleftrightarrow\quad
\sum_{i\in Y}x_i=\tfrac12\sum_{i\in X}x_i .
\tag{$\ddagger$}
\]

\begin{lemma}[Zero-padding encodes exact cardinality]\label{apx:lem:padding} Map a \textsc{Partition} instance $(n,u)$ to the \textsc{BalancedPartition} instance $(2n,w)$ with $w_i=u_i$ for $i\le n$ and $w_i=0$ otherwise. The map is polynomial-time, and $(n,u)$ admits an equal split exactly when $(2n,w)$ admits a balanced one.
\end{lemma}
\begin{proof}
Write $U=\sum_{i\le n}u_i$. First suppose some $S\subseteq[n]$ splits the $u$-sum equally, with $\abs S=k$. By $(\ddagger)$ it has $\sum_{i\in S}u_i=U/2$, so $A=S\cup\{n+1,\dots,2n-k\}$ has $\abs A=n=m/2$ and $w$-sum $U/2$, and $(2n,w)$ admits a balanced equal split. Conversely, suppose some $A$ with $\abs A=n$ splits the $w$-sum equally. Then $\sum_{i\in A}w_i=U/2$, and since the padded coordinates carry weight zero, $S=A\cap[n]$ splits the $u$-sum equally by $(\ddagger)$. 
\end{proof}
\noindent \textsc{BalancedPartition} lies in \textsf{NP}, since a verifier checks $\abs A=m/2$ and that $\sum_{i\in A}w_i$ equals half the total, so it is \textsf{NP}-complete.  Given a \textsc{BalancedPartition} instance $(m,w)$, we place the $m$ integers on the line, one key at $v_i=w_i$ with $d=1$, give every key the same mass $1/m$, and set the budget to $s=m/2$. Formally,
\[
\sigma(m,w)=(N,d,p,v,s)=\big(m,\,1,\,(\tfrac1m,\dots,\tfrac1m),\,w,\,\tfrac m2\big),
\]
a polynomial-time map.

\begin{lemma}\label{apx:lem:uniform}
With $W=\sum_iw_i$, the instance $\sigma(m,w)$ has $\mu=W/m$ and, for every $A\subseteq[m]$,
\[g_A=\tfrac1{m^{2}}\big(m\sum_{i\in A}w_i-\abs A\,W\big).\] At $\abs A=m/2$, $g_A=0$ exactly when $\sum_{i\in A}w_i=\sum_{i\notin A}w_i$.
\end{lemma}

\begin{proof}
We expand $g_A=\sum_{i\in A}\tfrac1m(w_i-\tfrac Wm)$. At $\abs A=m/2$ we factor out $m$ and apply $(\ddagger)$.
\end{proof}

\begin{lemma}\label{apx:lem:gap}
For every instance $\sigma(m,w)$ and every $A$ with $\abs A=m/2$, we have $m^2g_A\in\mathbb Z$, so $g_A=0$ or $\abs{g_A}\ge m^{-2}$. Also, $p(A)=\tfrac12$, so $\es(A)=2\abs{g_A}$. Hence $\ES(N/2)\in\{0\}\cup[2N^{-2},\infty)$ on every constructed instance.
\end{lemma}

\begin{proof}
Since $w$ has integer entries, $m^2g_A=m\sum_{i\in A}w_i-\abs A\,W$ is an integer for every $A$ with $\abs A=m/2$, so $g_A=0$ or $\abs{g_A}\ge m^{-2}$. Since $p(A)=\abs A/m=\tfrac12$, Lemma~\ref{lem:identities} gives \[\es(A)=\abs{g_A}/p(A)=2\abs{g_A}\] so each $\es(A)$ is $0$ or $\geq 2m^{-2}$. Then, $\ES(m/2)$ is the minimum of these, so it lies in $\{0\}\cup[2N^{-2},\infty)$ with $N=m$.
\end{proof}

\begin{proof}[Proof of Proposition~\ref{prop:hard}]
Given a subset $A$, a verifier checks $\abs A=s$ and tests $g_A=0$ in exact rational arithmetic. Clearing denominators turns the test into $d$ comparisons of integers with zero, and these integers have bit-length polynomial in the input size. An instance with $\ES(s)=0$ has such a subset, so \textsc{ZeroES} lies in \textsf{NP}.

Next, \textsc{BalancedPartition} is \textsf{NP}-complete (Lemma~\ref{apx:lem:padding}), and $\sigma$ reduces it to \textsc{ZeroES}. If $A$ is a balanced equal split, then $g_A=0$ by Lemma~\ref{apx:lem:uniform}, so $\es(A)=0$ and $\ES(m/2)=0$. Conversely, if $\ES(m/2)=0$, the minimum is attained, so some size-$m/2$ subset has $g_A=0$, and by Lemma~\ref{apx:lem:uniform} that subset splits the sum equally. Every constructed instance has $d=1$, uniform masses, integer points and $s=N/2$.

Now suppose an estimate $\widehat E$ satisfies $\ES(s)\le\widehat E\le f(I)\cdot\ES(s)$ for some function $f\ge1$. When $\ES(s)=0$ this forces $\widehat E=0$, and when $\ES(s)>0$ it forces $\widehat E>0$. Accepting exactly when $\widehat E=0$ therefore decides \textsc{ZeroES} ($f$ is never evaluated). Next suppose $\abs{\widehat E-\ES(s)}<N^{-2}$ on the constructed instances. There $\ES(s)$ is either $0$ or at least $2N^{-2}$ (Lemma~\ref{apx:lem:gap}), so $\widehat E<N^{-2}$ in the first case and $\widehat E>N^{-2}$ in the second. Accepting exactly when $\widehat E<N^{-2}$ therefore decides \textsc{ZeroES}. In both cases a polynomial-time estimate decides an \textsf{NP}-complete problem, so \textsf{P}$=$\textsf{NP} follows.
\end{proof}

\section{Frozen-Head Measurement Details}\label{app:protocol}
This appendix details the evaluation protocol and candidate-pool bounds behind \S\ref{sec:perhead}. Because $\ES_{\mathrm{pool}}$ minimizes error over a finite candidate set rather than all $\binom{N}{s}$ subsets, we bound how far $\ES_{\mathrm{pool}}$ sits above the exact optimum $\ES$ in two ways: by expanding the candidate set to test whether additional keys change the gap, and by running an exact full-key solver that either certifies the $\ES_{\mathrm{pool}}$ minimizer or finds a better subset. Throughout, we also report $\hat\kappa_{\mathrm{bal}}=\es(A_{\mathrm{bal}})/\ES_{\mathrm{pool}}$ (a lower bound on $\es(A_{\mathrm{bal}})/\ES$) and the one-pass $O(Nd)$ statistic $\Phi=\Vert g_{A^\star_s}\Vert$ from Lemma~\ref{lem:identities}.

\subsection{Setup}

Models are loaded in bfloat16 with eager attention on non-overlapping WikiText-103 test chunks (tokenized once). We record the per-head tensors $(P,V,K)$ in float32 and run all subset-selection arithmetic in float64; on every sampled query row, the recomputed dense output $\mu$ matches the model's forward pass within $10^{-5}$ on average and within $4\times10^{-3}$ in the worst case.

Each instance is a single sampled query row with causal attention weights $p=P[i,{:}i{+}1]$ and values $v=V[{:}i{+}1]$. We sample six query rows per head at every layer across budgets $s\in\{4,8,16,32\}$ over twenty $4{,}096$-token documents on each of the four main models: Qwen2.5-1.5B ($28\times12$ layers${}\times{}$heads), Llama-3.2-1B ($16\times32$), OLMo-2-1B ($16\times16$), and Qwen3-1.7B ($28\times16$). The first three documents per model ($111{,}744$ instances) form the primary evaluation set.

We also evaluate six additional models, bringing the total to ten (Table~\ref{tab:certificate}). Four of these (Qwen3-0.6B, Qwen3-4B, gemma-2-2b, and gemma-3-1b-pt) follow the same $4{,}096$-token protocol, using each Gemma model's native sliding-window attention. The other two are larger models: Qwen2.5-7B at $16{,}384$ tokens ($56{,}448$ instances) and Qwen3-14B at $4{,}096$ tokens. Finally, to test across text domains, we evaluate Qwen2.5-1.5B on English prose, source code, and French text ($72{,}576$ instances; Table~\ref{app:tab:content}).

\begin{table}[H]\centering\footnotesize
\renewcommand{\arraystretch}{0.98}
\begin{tabular}{l c | c | c | c | c | c}
Model & QK-norm & $\mathrm{AUC}(\bar m_s)$ & edge/int $[95\%]$ & $\rho_{\mathrm{depth}}$ & $\rho(\Phi,\mathrm{ES}_{\mathrm{pool}})$ & $\mathrm{med}\,\hat\kappa_{\mathrm{bal}}$ \\
\hline
Qwen2.5-1.5B & no  & $0.83$ & $1.89\ [1.51,2.29]$ & $0.80$ & $0.91$ & $1.00$ \\
Llama-3.2-1B & no  & $0.80$ & $1.98\ [1.72,2.26]$ & $0.67$ & $0.95$ & $1.00$ \\
Qwen3-1.7B   & yes & $0.85$ & $1.21\ [0.98,1.47]$ & $0.94$ & $0.98$ & $1.00$ \\
OLMo-2-1B    & yes & $0.76$ & $0.74\ [0.62,0.87]$ & $0.51$ & $0.80$ & $1.00$ \\
\hline
Qwen3-0.6B      & yes         & $0.85$ & $1.31\ [1.04,1.60]$ & $0.92$ & $0.98$ & $1.00$ \\
Qwen3-4B        & yes         & $0.76$ & $1.26\ [1.16,1.37]$ & $0.64$ & $0.97$ & $1.00$ \\
gemma-2-2b      & no$^{\dag}$ & $0.89$ & $1.13\ [0.86,1.41]$ & $0.97$ & $0.82$ & $1.00$ \\
gemma-3-1b-pt   & yes$^{\dag}$& $0.82$ & $1.25\ [0.95,1.56]$ & $0.79$ & $0.90$ & $1.00$ \\
\hline
Qwen2.5-7B (16k) & no & $0.85$ & $1.67\ [1.47,1.87]$ & $0.84$ & $0.91$ & $1.00$ \\
Qwen3-14B       & yes         & $0.80$ & $1.19\ [1.11,1.26]$ & $0.59$ & $0.96$ & $1.00$ \\
\end{tabular}
\caption{The ten-model panel consists of the four main models, four additional
models at $4{,}096$ tokens, and two larger models. $\mathrm{AUC}(\bar m_s)$ is evaluated at $s=8$ (document-cluster
$95\%$ confidence intervals within $\pm0.008$; the Qwen2.5-7B row is pooled across budgets); ``edge/int'' is the rate of $\hat\kappa>1.11$ in the
first and last layers relative to interior layers, with head-cluster $95\%$ intervals;
$\rho_{\mathrm{depth}}$ is the Spearman rank correlation across layers between median dropped mass
and the $\hat\kappa>1.11$ rate; $\rho(\Phi,\mathrm{ES}_{\mathrm{pool}})$ is the pooled Spearman correlation of $\Phi$
with $\mathrm{ES}_{\mathrm{pool}}$; and $\mathrm{med}\,\hat\kappa_{\mathrm{bal}}$ is the balancing selector's
median ratio ($95\%$ intervals $[1.00,1.00]$,
\S\ref{sec:ceiling}). $^{\dag}$gemma-2-2b uses logit soft-capping and
$4{,}096$-token interleaved sliding windows instead of QK-norm; gemma-3-1b-pt uses QK-norm with $512$-token windows.}
\label{tab:certificate}
\end{table}

\subsection{Enumeration}

Because $\ES_{\mathrm{pool}}\ge\ES$ on every instance, both $\hat\kappa=\es(\mass_s)/\ES_{\mathrm{pool}}$ and the closable share $\hat c=1-\ES_{\mathrm{pool}}/\es(\mass_s)$ are lower bounds on $\kappa$ and $c$. Thus, instances with $\hat\kappa>1.11$ contain no false positives, whereas $\hat c$ can understate $c$. We compute $\ES_{\mathrm{pool}}(s)$ by enumerating the candidate set of Definition~\ref{def:pool} ($24/12/6/4$ keys beyond $\mass_s$ at $s=4/8/16/32$); keys with both small weight and small value deviation move neither $p(A)$ nor $m_A$ and cannot reduce $\es$.

Expanding the candidate search by $25\times$ at the largest budget $s=32$---from $2\cdot10^{5}$ to $5\cdot10^{6}$ subsets (four to six keys beyond $\mass_s$)---changes the median gap by under $1\%$, and across the sampled heads the fraction of instances with $\hat\kappa>1.11$ edges up from $35.6\%$ to $37.8\%$ and then $38.7\%$. To test whether subsets outside the candidate pool can improve on $\ES_{\mathrm{pool}}$, we also run an exact full-key solver over all $N$ keys on a stratified $960$-instance audit set sampled evenly across the four main models, four budgets, and dropped-mass deciles. We formulate $\ES(s)\le\tau$ as a mixed-integer second-order-cone feasibility problem and bisect on $\tau$ using the $\ES_{\mathrm{pool}}$ minimizer as a warm start (SCIP, $120$\,s limit per instance).

Of the $960$ instances, $939$ complete within memory limits ($21$ hit memory limits). The solver closes the bisection bracket on $\ES$ to a relative width of $10^{-3}$ on $159$ of the $939$ completed instances ($17\%$) and finds a subset more than $1\%$ below $\ES_{\mathrm{pool}}$ on $46$ instances ($4.9\%$, with $36$ at $s=32$ and a maximum reduction of $75\%$). Although SCIP flags $347$ instances as optimal, on $188$ of them the returned subset does not reach the upper endpoint of the relative $10^{-3}$ bracket because SCIP checks feasibility up to an absolute numerical tolerance on $\Vert g\Vert^2$. Consistent with an absolute tolerance, the unclosed fraction among these $347$ instances drops from $0.97$ to $0.24$ from the smallest to the largest quintiles of $\Vert g\Vert^2$ (Table~\ref{app:tab:solver}).

\begin{table}[H]\centering\small
\renewcommand{\arraystretch}{1.05}
\begin{tabular}{c | c | c c c | c | c}
$s$ & solved & closed & open, flagged optimal & open, time limit & memory failure & improved \\
\hline
$4$  & $237$ & $29$  & $68$  & $140$ & $3$  & $0$ \\
$8$  & $234$ & $37$  & $49$  & $148$ & $6$  & $2$ \\
$16$ & $234$ & $41$  & $40$  & $153$ & $6$  & $8$ \\
$32$ & $234$ & $52$  & $31$  & $151$ & $6$  & $36$ \\
\hline
all  & $939$ & $\mathbf{159}$ & $188$ & $592$ & $21$ & $46$ \\
\end{tabular}
\caption{Exact full-key solver status on the $960$-instance audit set. A bracket is closed when the best returned subset lies within a relative $10^{-3}$ of the proven lower bound; ``improved'' counts instances where the solver found a subset more than $1\%$ below $\ES_{\mathrm{pool}}$ ($29$ of which are among the closed brackets). The closed-bracket fraction ranges from $0.11$ (Llama-3.2-1B) to $0.21$ (Qwen2.5-1.5B and Qwen3-1.7B) across models, and $140$ of the time-limited instances have a proven positive lower bound.}
\label{app:tab:solver}
\end{table}

Across the eight $4{,}096$-token models at $s\in\{4,8\}$, $\hat\kappa\le1.5$ on $89\%$ to $90\%$ of instances and $\hat\kappa\le2$ on $97\%$, with top-mass matching $\ES_{\mathrm{pool}}$ exactly on $39\%$ of instances at $s=4$ and $21\%$ at $s=8$. At $s=4$, the per-model fraction with $\hat\kappa\le1.5$ ranges from $0.75$ (OLMo-2-1B) to $0.95$ (Qwen3-0.6B), and every model exceeds $0.91$ at $\hat\kappa\le2$. Within this subset, the median $\hat\kappa$ is $1.00$ to $1.03$ and the $90$th percentile is $1.14$ to $1.28$ (because $\ES_{\mathrm{pool}}\ge\ES$, coverage fractions are upper bounds and $\hat\kappa$ quantiles are lower bounds).

\subsection{Penalty}\label{app:penalty}
The cost of forced weights grows with the budget $s$. The ratio $\pi(s)=\ES(s)/\EF(s)\ge1$ compares the best subset under forced weights against the free-weight optimum $\EF(s)$. We estimate $\pi(s)$ by $\hat\pi=\ES_{\mathrm{pool}}/\EF_{\mathrm{FW}}$, where $\EF_{\mathrm{FW}}$ is the Frank--Wolfe~\citep{jaggi} estimate of $\EF(s)$; because both numerator ($\ES_{\mathrm{pool}}\ge\ES$) and denominator ($\EF_{\mathrm{FW}}\ge\EF$) are upper bounds, we report $\hat\pi$ as a point estimate. On the reference model (Qwen2.5-1.5B), the median of $\hat\pi$ increases from $1.07$ to $1.79$ ($99$th percentile from $2.2$ to $6.8$) as $s$ grows from $4$ to $32$, and at $s=32$ the per-model median ranges from $1.79$ to $2.39$. Across every model, median $\hat\pi$ increases with $\bar m_s\sqrt s$. Moreover, $\bar m_s$ separates instances with $\hat\kappa>1.11$ from the rest: the rank correlation of $\mathbf 1\{\hat\kappa>1.11\}$ with $\bar m_s$ is $0.42$ to $0.46$ on every model, whereas its correlation with the value diameter $D$ is weak and changes sign across models. Even on a model with massive-activation outliers~\citep{spikesink} ($D\approx500$ in deep layers), those layers have small dropped mass ($\bar m_s\approx0.07$) and almost no instances with $\hat\kappa>1.11$.

When evaluated under renormalized softmax weights, purely geometric subsets also underperform top-mass (Table~\ref{app:tab:perbudget}). Let $A^{\mathrm{geo}}_s$ be the support obtained by minimizing $\mathrm{ef}$ via Frank--Wolfe (padded to size $s$ with top-mass keys), and let $A^{\mathrm{lev}}_s$ be the $s$ keys of largest ridge leverage $k_j^{\top}(K^{\top}K+\lambda I)^{-1}k_j$. As $s$ increases to $32$, the median error ratio $\mathrm{es}(A^{\mathrm{geo}}_s)/\mathrm{es}(A^\star_s)$ rises to $2.06$ and $\mathrm{es}(A^{\mathrm{lev}}_s)/\mathrm{es}(A^\star_s)$ rises to $20.3$ (Table~\ref{app:tab:perbudget}). By contrast, the balancing selector matches $\ES_{\mathrm{pool}}$ in median across every model and budget after a median of $1$ to $2$ swaps.

\subsection{Solver Substitution}\label{app:cert}
Because both the balancing selector and $\ES_{\mathrm{pool}}$ are restricted to the candidate set, a subset outside that set could in principle beat both. Across the $960$-instance full-key audit, the solver found $46$ instances where a full-key subset beat $\ES_{\mathrm{pool}}$ by more than $1\%$ ($0/2/8/36$ at $s=4/8/16/32$, with a median improvement of $3.8\%$ and a maximum of $75\%$; Table~\ref{app:tab:cert}). Replacing $\ES_{\mathrm{pool}}$ with the solver's best subset on those instances shifts the audit-set median $\hat c$ by at most $+0.0063$ (from $0.051$ to $0.057$ at $s=32$) and median $\hat\kappa_{\mathrm{bal}}$ by $3\times10^{-7}$; the fraction of instances where the best found subset beats balancing by more than $1\%$ rises from $1.3\%$ to $5.1\%$ at $s=32$ and remains below $5\%$ at smaller budgets. More generally, increasing at most a fraction $f$ of a sample cannot raise its median above its original $(0.5+f)$-quantile. Using the $95\%$ Wilson upper bound $f_{\mathrm{hi}}$ on the improvement rate therefore yields a distribution-free upper bound on the population median of $c$ (at most $0.108$ at $s=32$), assuming unclosed brackets hide no additional improvements. Conservatively counting the $21$ memory-limited instances as improved raises this median bound to $0.027\text{--}0.119$, and counting the $188$ unclosed optimal-flagged instances as well raises it to $0.176\text{--}0.266$ (counting all $592$ time-limited instances exceeds $50\%$ of the sample and leaves the quantile bound unconstrained). By contrast, the median of $\hat\kappa_{\mathrm{bal}}$ remains $1.00$ under all of these substitutions because $\hat\kappa_{\mathrm{bal}}=1$ on $87.0\%$ to $94.0\%$ of instances per budget (at least $17\%$ above the $50\%$ median threshold in every model-by-budget cell).

\begin{table}[H]\centering\small
\renewcommand{\arraystretch}{1.05}
\begin{tabular}{c | c c | c c c c}
$s$ & improved & Wilson-$95$ $f_{\mathrm{hi}}$ & median & cap at $f_{\mathrm{hi}}$ & cap, unsolved & cap, open optimal \\
\hline
$4$  & $0/237$  & $0.016$ & $0.018$ & $0.022$ & $0.027$ & $0.266$ \\
$8$  & $2/234$  & $0.031$ & $0.040$ & $0.047$ & $0.057$ & $0.198$ \\
$16$ & $8/234$  & $0.066$ & $0.047$ & $0.064$ & $0.073$ & $0.176$ \\
$32$ & $36/234$ & $0.206$ & $0.046$ & $\mathbf{0.108}$ & $0.119$ & $0.201$ \\
\end{tabular}
\caption{Worst-case bounds on median closable share $c$ under full-key solver substitution. ``improved'' counts audit instances where the solver beat $\ES_{\mathrm{pool}}$ by more than $1\%$; $f_{\mathrm{hi}}$ is the $95\%$ Wilson upper bound on that rate; ``median'' is the full-sample median of $\hat c$ across the four main models ($186{,}240$ instances per budget). The rightmost three columns report distribution-free $(0.5+f)$-quantile upper bounds on the median after substitution: first using $f_{\mathrm{hi}}$, then counting the $21$ memory-limited instances as improved, and finally counting the $188$ unclosed optimal-flagged brackets as improved as well (counting the $592$ time-limited brackets exceeds $50\%$ of the audit sample). Under all three substitutions, the corresponding upper bound on median $\hat\kappa_{\mathrm{bal}}$ is $1.00$ at every budget.}
\label{app:tab:cert}
\end{table}

Across budgets $s\in\{4,8,16,32\}$, the $95\%$ confidence intervals for the pooled median $\hat c$ lie within $[0.017,0.051]$. Although exhaustive candidate enumeration stops at $s=32$, the fraction of top-mass error removed by the balancing selector, $1-\es(A_{\mathrm{bal}})/\es(\mass_s)$, provides a direct lower bound on $c$ at any budget. Evaluating this bound through $s=256$ on Qwen2.5-1.5B and OLMo-2-1B (Table~\ref{app:tab:anchor}) shows that the closable share stays near its small-budget level ($0.04$ to $0.13$ in median at $s\in\{64,128,256\}$, with balancing beating top-mass by at least $10\%$ on $28\%$ to $58\%$ of rows), even as the median dropped mass $\bar m_s$ falls roughly tenfold.

\begin{table}[H]\centering\footnotesize
\renewcommand{\arraystretch}{1.05}
\begin{tabular}{l | c c c c | c c c}
median balancing closure & $s=4$ & $8$ & $16$ & $32$ & $64$ & $128$ & $256$ \\
\hline
Qwen2.5-1.5B   & $0.001$ & $0.023$ & $0.032$ & $0.032$ & $0.044$ & $0.047$ & $0.041$ \\
OLMo-2-1B      & $0.134$ & $0.144$ & $0.124$ & $0.091$ & $0.126$ & $0.127$ & $0.105$ \\
\hline
pooled         & $0.032$ & $0.057$ & $0.058$ & $0.050$ & $\mathbf{0.075}$ & $\mathbf{0.077}$ & $\mathbf{0.061}$ \\
median dropped mass & $0.38$ & $0.28$ & $0.20$ & $0.13$ & $0.08$ & $0.04$ & $0.02$ \\
gap, share of its $s=8$ value & --- & $1$ & --- & --- & $0.28/0.34$ & $0.16/0.21$ & $0.08/0.12$ \\
\end{tabular}
\caption{Balancing closure $1-\es(A_{\mathrm{bal}})/\es(\mass_s)$ (a lower bound on $c$) from $s=4$ to $256$ (the vertical line marks the $s=32$ enumeration limit). The relative share closed by balancing remains stable through $s=256$, while the absolute error gap---bounded by the dropped mass $\bar m_s$---shrinks roughly tenfold. Across all budgets, balancing reduces top-mass error by at least $10\%$ on $21\%$ to $30\%$ of rows on Qwen2.5-1.5B and $47\%$ to $59\%$ on OLMo-2-1B. The bottom row reports the remaining gap at $s\in\{64,128,256\}$ as a fraction of its $s=8$ value on Qwen2.5-1.5B and OLMo-2-1B (\S\ref{sec:audit}).}
\label{app:tab:anchor}
\end{table}

\subsection{Reference tables}\label{app:perheadtables}
\begin{figure}[H]\centering
\includegraphics[width=\textwidth]{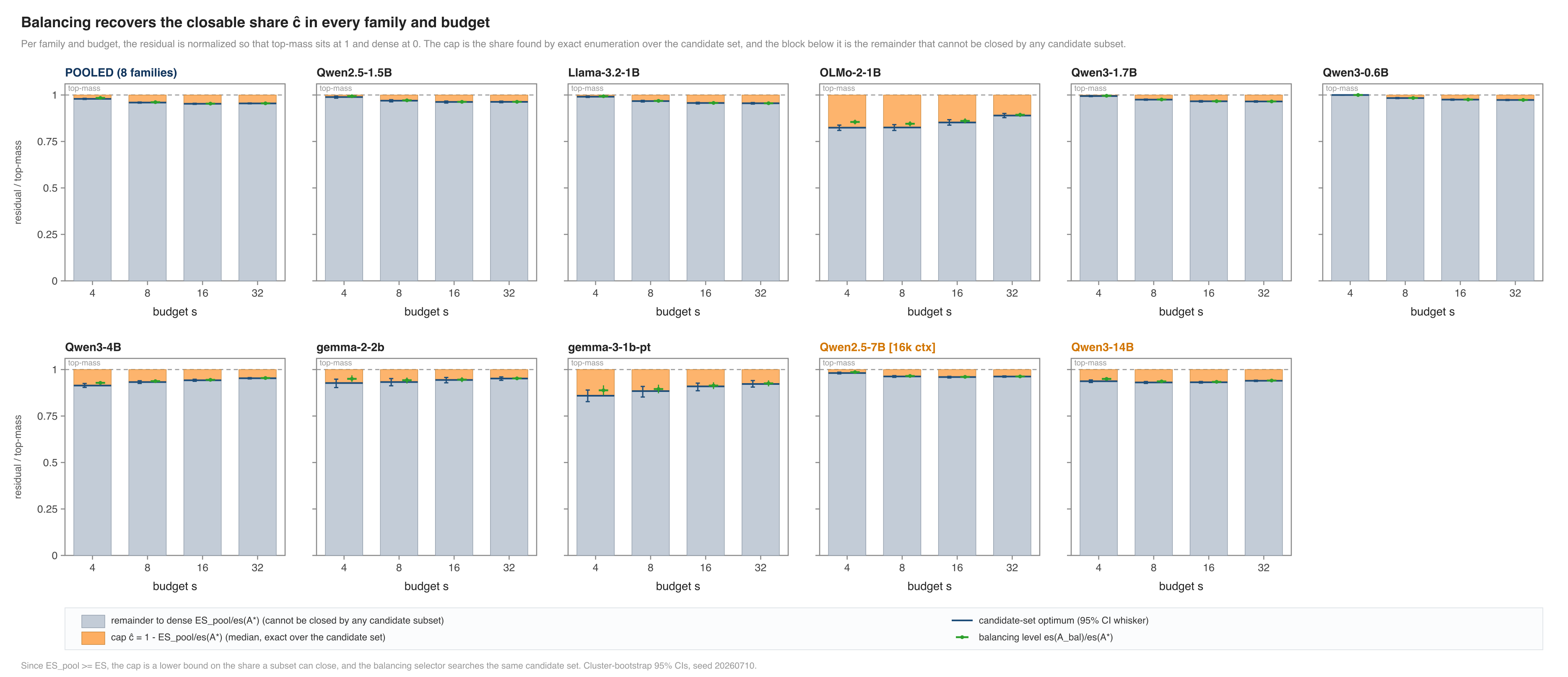}
\caption{\textbf{Ceiling across 10 models.} Per-model decomposition matching Figure~\ref{fig:ceiling}, with the two larger models shown last. Per-model median $\hat c$ ranges from $0.000$ (Qwen3-0.6B at $s=4$) to $0.176$ (OLMo-2-1B at $s=4$). In every model-by-budget cell, the balancing selector's median sits within $0.031$ of the pool optimum (with the largest offsets at $s=4$), and its median recovery of the closable gap is $1.00$ everywhere.}
\label{fig:ceilingfull}
\end{figure}
The one-pass $O(Nd)$ statistic $\Phi=\Vert g_{A^\star_s}\Vert$ of Lemma~\ref{lem:identities} tracks $\ES_{\mathrm{pool}}$ at Spearman $0.84$ to $0.98$ across budgets on the reference model (Qwen2.5-1.5B) and $0.80$ to $0.98$ pooled per model (Tables~\ref{app:tab:perbudget} and~\ref{tab:certificate}).
\begin{table}[H]\centering\small
\renewcommand{\arraystretch}{1.05}
\begin{tabular}{c | ccc | cc | cc | c}
$s$ & $\mathrm{med}\,\hat\kappa$ & $p99\,\hat\kappa$ & share &
$\dfrac{\es(A^{\mathrm{geo}}_s)}{\es(\mass_s)}$ & $\dfrac{\es(A^{\mathrm{lev}}_s)}{\es(\mass_s)}$ &
$\mathrm{med}\,\hat\pi$ & $p99\,\hat\pi$ & $\rho(\Phi,\ES_{\mathrm{pool}})$ \\[2pt]
\hline
$4$  & $1.003$ & $2.28$ & $23.5\%$ & $1.40$ & $6.8$  & $1.07$ & $2.18$ & $0.84$ \\
$8$  & $1.025$ & $2.27$ & $25.9\%$ & $1.73$ & $9.7$  & $1.16$ & $2.90$ & $0.92$ \\
$16$ & $1.034$ & $2.15$ & $26.5\%$ & $1.96$ & $13.8$ & $1.39$ & $4.27$ & $0.96$ \\
$32$ & $1.033$ & $1.91$ & $23.6\%$ & $\mathbf{2.06}$ & $\mathbf{20.3}$ & $\mathbf{1.79}$ & $6.75$ & $0.98$ \\
\end{tabular}
\caption{Per-budget profile on Qwen2.5-1.5B ($6{,}048$ instances per budget). ``share'' is the fraction of instances with $\ES_{\mathrm{pool}}\le0.9\,\es(\mass_s)$. $A^{\mathrm{geo}}_s$ is the Frank--Wolfe support padded to size $s$ with top-mass keys, and $A^{\mathrm{lev}}_s$ is the ridge-leverage subset. Across all models at $s=32$, the median geometric ratio ranges from $1.3$ to $2.6$ and the leverage ratio from $7$ to $32$.}
\label{app:tab:perbudget}
\end{table}

\begin{table}[H]\centering\small
\renewcommand{\arraystretch}{1.05}
\begin{tabular}{l | c | c}
content type & $\mathrm{AUC}(\bar m_s)$ span & share \\
\hline
English prose  & $0.78$ to $0.85$ & $26$ to $30\%$ \\
source code    & floor of the span, $0.71$ & $33$ to $37\%$ \\
French text    & $0.81$ to $0.86$ & $20$ to $23\%$ \\
\end{tabular}
\caption{Content diversity on Qwen2.5-1.5B ($72{,}576$ instances). ``share'' is the fraction of instances with $\hat\kappa>1.11$. The $\mathrm{AUC}(\bar m_s)$ values remain high across all three domains, while code exhibits the largest headroom (median $\hat\kappa$ up to $1.06$), consistent with flatter attention distributions leaving a heavier tail. Document-level $95\%$ confidence intervals are approximately $\pm5$ percentage points on the share and $\pm0.01$ to $\pm0.035$ on the AUC.}
\label{app:tab:content}
\end{table}

Although each query row is evaluated as an individual instance in \S\ref{sec:perhead}, grouped-query attention (GQA) shares a single KV head---and therefore a single kept subset of size $s$---across a group of query heads and future query rows. Because any shared subset is also a valid candidate for each individual row, the per-row optimum $\ES(s)$ lower-bounds the error of any group-shared rule. Selecting a shared subset by group vote (keeping the $s$ keys with the largest attention mass summed across the group's query rows) increases per-row error at $s=8$ by a median of $1.06\times$ to $1.22\times$ relative to each row's own top-mass subset ($1.06\times$, $1.12\times$, $1.17\times$, and $1.22\times$ at group sizes $2$, $4$, $5$, and $6$, respectively) and by $1.14\times$ to $1.38\times$ relative to each row's own $\ES_{\mathrm{pool}}$. Group voting increases error on $76\%$ to $90\%$ of rows relative to per-row top-mass while reducing error on $7\%$ to $10\%$ (when another head's vote retains a key that improves the row's balance). On OLMo-2-1B (multi-head attention, group size $1$), the group vote coincides with per-row top-mass.

\begin{figure}[H]\centering
\includegraphics[width=0.99\textwidth]{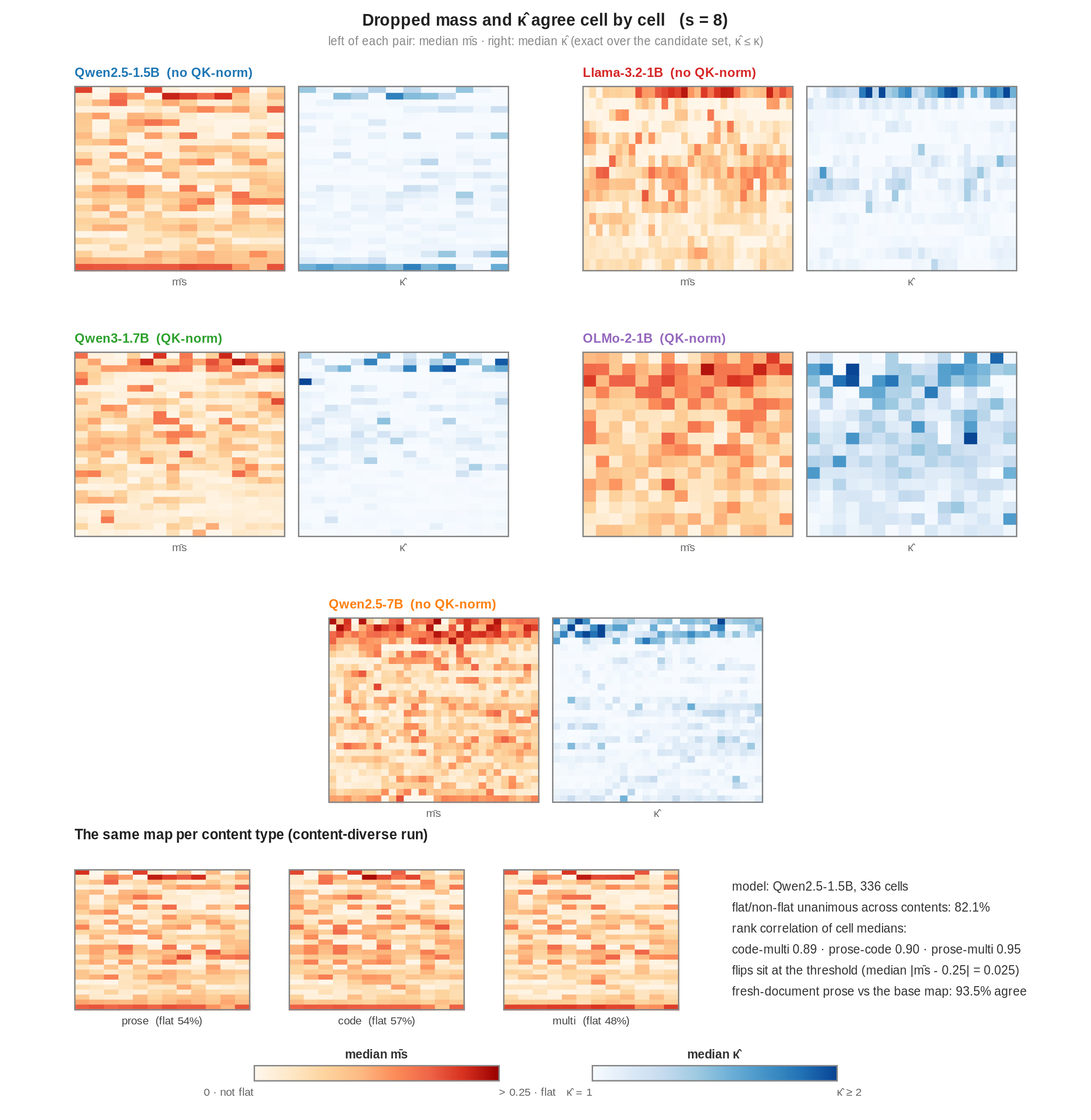}
\caption{\textbf{The map.} For each (layer, head) cell at $s=8$, the left panel shows the median dropped mass $\bar m_s$ and the right panel shows the cell median $\hat\kappa$. Cell-level AUC for predicting median $\hat\kappa\ge 1.11$ from $\bar m_s$ is $0.82$ to $0.94$, and the balancing selector closes the gap in every cell whose median $\hat\kappa$ reaches $1.11$ (since $\hat\kappa\le\kappa$, those cells also have true median $\kappa\ge 1.11$).}
\label{fig:map}
\end{figure}

\paragraph{Task-structured contexts.} These findings also hold on synthetic benchmarks designed to concentrate retrieval: planted passkey needles~\citep{passkey}, multi-hop variable chains, and long-document QA, constructed in token space so that target positions are exact and retrieval rows account for at least $20\%$ of instances. Across these contexts, the pooled median closable share $\hat c$ remains between $0.016$ and $0.042$ across budgets, and the $90$th-percentile headroom ratio of retrieval-critical rows to background rows is $0.83$ to $0.97$. At $s=8$, only $1.87\%$ of retrieval rows require value-aware selection ($\hat\kappa>1.11$), and $\bar m_s$ predicts these instances with AUC $0.75$ to $0.80$. The per-cell regime classifications from Figure~\ref{fig:map} (flat vs.\ non-flat) also transfer across these tasks, with label agreement $0.90$ to $0.93$ and rank correlation $+0.94$ to $+0.97$.

\subsection{Other Details}
\paragraph{Cross-head budget allocation.} Reallocating budget across heads according to the per-cell map in Figure~\ref{fig:map} increases pooled reconstruction error on every model compared to a uniform budget at matched total memory, with $\Delta\es$ of $+0.044$ (Llama-3.2-1B), $+0.144$ (OLMo-2-1B), $+0.255$ (Qwen2.5-1.5B), and $+1.032$ (Qwen3-1.7B) at an average budget of $s=8$ under the balancing selector. Cell-cluster $95\%$ confidence intervals exclude zero, and the increase persists at average budget $s=16$ and under top-mass selection, alongside mass-greedy~\citep{adakv} and pyramid-shaped~\citep{pyramidkv} allocators. Because flat cells have a small closable gap (\S\ref{sec:ceiling}), shifting extra budget into or out of them yields little reduction in error while starving heads with larger dropped mass.

\paragraph{Absolute forms.} An $O(Nd^2)$ spectral bound does not outperform $\bar m_s$ on any model. Both prior-art bounds classify $\mathbf 1\{\hat\kappa>1.11\}$ with lower AUC than $\bar m_s$ alone on every model (AUC $0.58\text{--}0.87$ and $0.58\text{--}0.86$, vs.\ $0.76\text{--}0.89$ for $\bar m_s$). While the diameter bound $2D\bar m_s$ and the variance refinement $\sqrt{\bar m_s/(1-\bar m_s)}\,\sigma$~\citep{tvtopk} correlate strongly with the absolute top-mass error $\es(\mass_s)$ (Spearman $0.63\text{--}0.95$ and $0.89\text{--}0.98$), their rank correlation with the relative gap $\hat\kappa$ drops to $0.23\text{--}0.70$ and $0.21\text{--}0.68$ (compared to $0.48\text{--}0.73$ for $\bar m_s$ alone). The cross-model ordering at $s=8$ is invariant across threshold sweeps from $\hat\kappa\ge1.00$ to $1.43$ (Kendall $\tau=+1.0$), and relabeling instances using the exact solver's best subsets on the audit set yields per-model AUCs of $0.72$ to $0.87$, with confidence intervals overlapping the full-sample estimates.

\paragraph{Cost.} On an H100 GPU, computing $\bar m_s$ adds $O(s)$ work to top-mass selection, taking $0.002\times$ to $0.03\times$ the runtime of a dense attention layer across $N=2{,}048$ to $16{,}384$ ($0.12$\,ms per layer at $N=16{,}384$). One balancing swap round ($192$ drop-add pairs at $\Theta(d)$) costs roughly $2\%$ of a single row's $QK^{\top}$ compute at $N=4{,}096$, and the selector terminates in a median of $1$ to $2$ swaps across the $111{,}744$ primary instances. The balancing selector is practical at decode time, where $\mu$ is already computed for the current step and swaps run on a single query row; during batched prefill, running swaps across all rows is kernel-launch-bound ($20\times$ to $114\times$ a dense layer) and requires computing $\mu$ ($0.31\times$ to $0.43\times$ a dense layer). Profiling the per-head dropped-mass map requires one offline dense forward pass over sixteen documents and amortizes across serving.
\section{Selector Details}\label{app:audittables}
We score each rule's kept set against the enumerated optimum $\mathrm{ES}_{\mathrm{pool}}$
on the same $111{,}744$ instances as \S\ref{sec:audit}, joined row by row to the recorded
runs. SnapKV runs with window $\max(1,\min(32,s/2))$ and kernel $7$ with max pooling, H2O
with accumulated mass and recency $s/2$, Quest with pages of $\max(2,s/4)$ and a
page-$16$ configuration at $s\ge16$, and StreamingLLM with $\min(4,s/2)$ sinks plus
recency. TOVA's per-query keep set is identical to top-mass, and this identity is checked at
runtime, so we do not score it separately. Budgets are accounted in two ways: per-query,
where a rule re-chooses its set at every query, and cache-faithful, where one evicted
cache serves all queries, so the per-query reading is never above the cache-faithful one.
\begin{figure}[H]\centering
\includegraphics[width=0.9\linewidth]{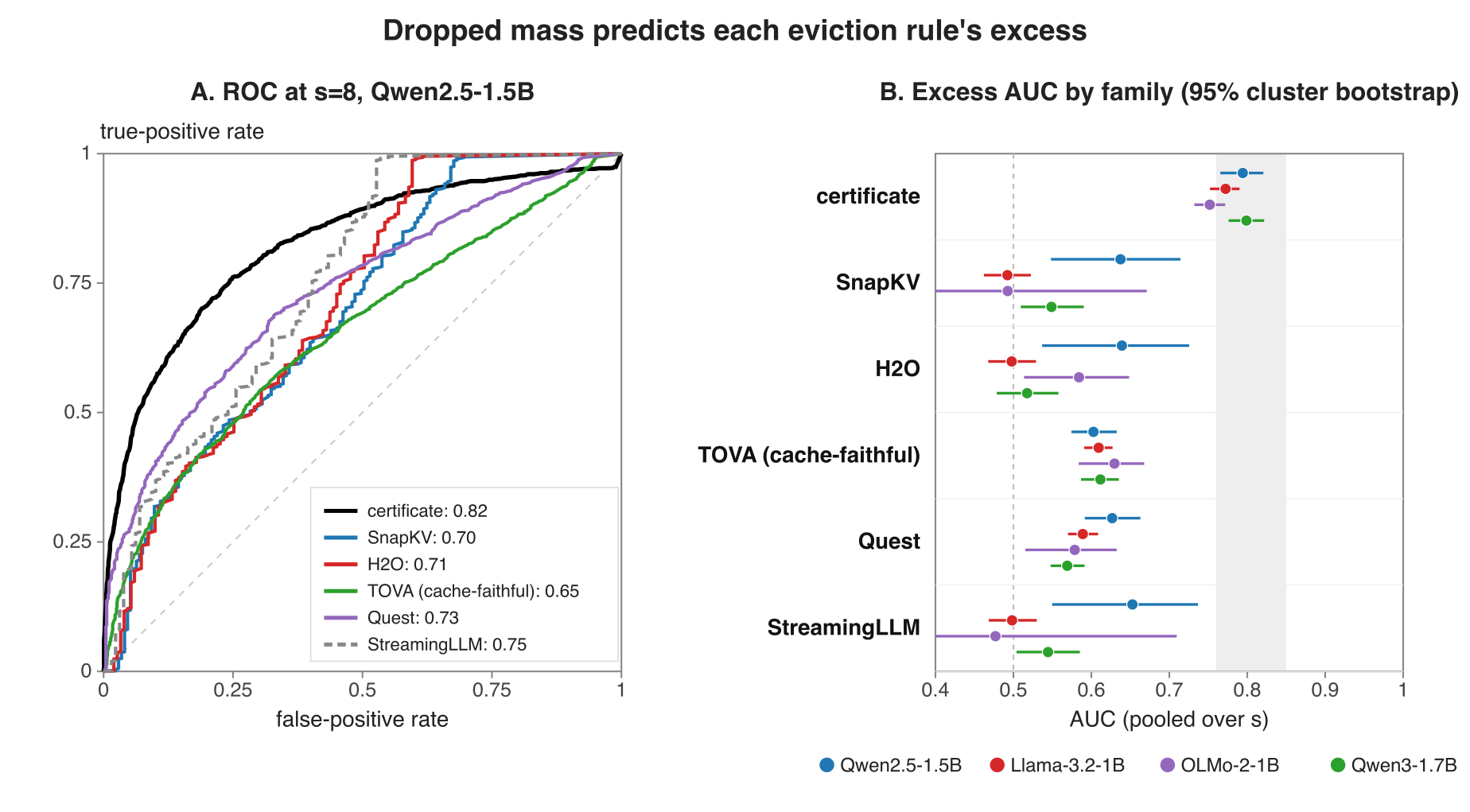}
\caption{The left panel shows the ROC of dropped mass as a predictor of each rule's excess label at $s=8$ on the reference family, and the right panel shows the same AUC pooled per family with $95\%$ intervals. The grey band is the AUC range of \S\ref{sec:certificate}, the dashed line is chance, and the row labeled ``certificate'' scores $\bar m_s$ itself.}
\label{fig:audit}
\end{figure}

\paragraph{Small budgets.} At $s=8$ the four scored rules sit at a median $1.4$ to $2.0$ times $\mathrm{ES}_{\mathrm{pool}}$ against $1.02$ to $1.03$ for top-mass on the three sharper families, and at $1.6$ to $4.7$ against $1.19$ on OLMo-2-1B. Top-mass exceeds $1.11\,\mathrm{ES}_{\mathrm{pool}}$ on as few as $19\%$ of instances, while each deployed rule exceeds it on $75$ to $99\%$ of instances at $s\le32$, so the label $\mathbf 1\{\mathrm{es}_{\mathrm{method}}>1.11\,\mathrm{ES}_{\mathrm{pool}}\}$ is almost always $1$ and its AUC, pooled across families, is $0.49$ to $0.64$ for mass rules and a positional control alike (Fig.~\ref{fig:audit}). At the deployed budgets the label is $1$ on fewer instances, and $\bar m_s$ predicts it, since three of the four mass-based selectors clear AUC $0.70$ at $s=128$ and all four clear it at $s=256$ on the reference family. Ranking each rule's excess $\mathrm{es}_{\mathrm{method}}/\mathrm{ES}_{\mathrm{pool}}$ by $\bar m_s$ reaches Spearman $0.70$ to $0.84$ on the families without query-key normalization and $0.36$ to $0.61$ on the two with it. SnapKV's error holds at $1.8$ to $2.2$ times the same-budget top-mass error through the deployed budgets. Cache-faithful accounting raises every rule's median excess, up to about $3.0$ for the sequential rules and $4.2$ for the frozen-set SnapKV reading at $s=8$. Quest is evaluated below the page sizes that it was designed for ($2$ to $8$ keys per page at $s\le32$). At $s\in\{64,128,256\}$ exact enumeration is infeasible, so the references are top-mass and the balancing subset, and OLMo-2-1B reaches AUC $0.70$ only at $s=256$, which is one budget later than the reference family.

\paragraph{Held-out loss.} At $s=8$ Quest raises held-out loss the least and StreamingLLM the most. Per-query loss stays at or below cache-faithful loss wherever both readings exist. On Qwen3-1.7B the flat-cell switch recovers $0.18$ to $0.65$ of the gap to all-cell balancing, which is consistent with \S\ref{sec:teeth} on the same family.

Switching only the flat cells to balancing lowers held-out continuation cross-entropy at $s=8$ by $0.122$ to $0.243$ nats on five of the eight families (four lineages, $1$B to $7$B, every interval below zero). We measure loss in the model forward at matched budget, where each head's attention is replaced with the restricted operator $m_A$, and the same code reproduces the unmodified model up to floating-point error. The switch captures $88$ to $98\%$ of what all-cell balancing achieves, and on Qwen2.5-7B at $16{,}384$ tokens the gap doubles while the same switch recovers $63\%$ of it. Keeping each rule on non-flat cells while switching the flat cells to balancing recovers $0.51$ to $0.93$ of its gap to all-cell balancing on all nine (selector, family) pairs, with each paired $95\%$ interval below zero. On the three Qwen3 scales the map selects the wrong cells and keeps $55$/$15$/$31\%$ of the all-cell gain, so whether the map converts at the cell level must be measured per family. A rank-$16$ scorer trained on the query and key projections to reproduce the balancing swaps~\citep{oracleprefill} recovers at most $5\%$ of the decode-time loss reduction against a $25\%$ bar fixed in advance, because the gain sits in the value-side residual $g_A$, which lies outside those projections (App.~\ref{app:scope}). The constant that converts a reconstruction gain into a loss gain varies by four orders of magnitude across families. It is smallest where the reconstruction gain is largest, and dropped mass does not predict it (App.~\ref{app:scope}).

\paragraph{Oracle share of the loss gap.} All-cell balancing closes $0.43$ to $0.60$ of the top-mass-to-dense cross-entropy gap at $s=8$ on seven families (Table~\ref{app:tab:celoss}), while its median share of the reconstruction gap is $0.016$ to $0.155$ on the same families. Balancing reads the dense output $\mu$, so this is an oracle figure. On Qwen3-4B the whole gap is $0.015$ nats and balancing goes below dense, so the share is not defined there.

\begin{table}[H]\centering\small
\renewcommand{\arraystretch}{1.05}
\begin{tabular}{l | c | c | c}
family & gap (nats) & balancing share $[95\%]$ & flat-cell switch / balancing \\
\hline
Qwen2.5-1.5B  & $0.249$ & $0.52\ [0.50,0.55]$ & $0.94$ \\
Llama-3.2-1B  & $0.330$ & $0.50\ [0.48,0.51]$ & $0.88$ \\
OLMo-2-1B     & $0.330$ & $0.58\ [0.57,0.59]$ & $0.89$ \\
Qwen2.5-7B    & $0.423$ & $0.55\ [0.49,0.60]$ & $0.90$ \\
gemma-3-1b-pt & $0.575$ & $0.43\ [0.41,0.45]$ & $0.98$ \\
Qwen3-0.6B    & $0.294$ & $0.50\ [0.48,0.52]$ & $0.55$ \\
Qwen3-1.7B    & $0.106$ & $0.60\ [0.53,0.68]$ & $0.15$ \\
Qwen3-4B      & $0.015$ & --- & $0.31$ \\
\end{tabular}
\caption{Held-out cross-entropy at $s=8$. ``gap'' is top-mass minus dense, and the shares are of that gap, with paired document-bootstrap intervals ($B=4000$). The last column is the fraction of the all-cell gain that the flat-cell switch keeps.}
\label{app:tab:celoss}
\end{table}

\paragraph{Deployed rules in loss.} At $s=8$, per-query SnapKV, H2O and StreamingLLM sit $1.15$ to $2.48$ nats above top-mass in all four headline families, and Quest sits $0.15$ to $1.01$ above it, while the whole gap from top-mass to dense is $0.11$ to $0.33$ nats.

\paragraph{Causal eviction.} In a causal one-cache task, each rule evicts once at the end of prefill, scoring from the final prompt row so that the question is visible, and then step by step during decoding. On passkey retrieval~\citep{passkey}, top-mass eviction, which is TOVA's rule, trails a running SnapKV by $0.53$ to $0.89$ in exact match at $s\ge128$ on Llama-3.2-1B and Qwen2.5-1.5B (Table~\ref{app:tab:causal}). On Qwen3-0.6B, Qwen3-1.7B and Qwen3-4B it stays within $0.12$ of the running SnapKV at $s\le64$ and leads it on Qwen3-4B at $s=64$ ($0.51$ against $0.39$). On TriviaQA and SAMSum, which the running SnapKV did not cover, top-mass eviction stays within $0.023$ of the best arm.

\begin{table}[H]\centering\small
\renewcommand{\arraystretch}{1.05}
\begin{tabular}{l l | c c c c}
model & rule & $s=32$ & $64$ & $128$ & $256$ \\
\hline
Llama-3.2-1B (dense $0.99$) & top-mass eviction & $0.03$ & $0.07$ & $0.10$ & $0.24$ \\
                            & running SnapKV    & $0.83$ & $0.93$ & $0.99$ & $0.99$ \\
\hline
Qwen2.5-1.5B (dense $0.96$) & top-mass eviction & $0.00$ & $0.03$ & $0.12$ & $0.37$ \\
                            & running SnapKV    & $0.02$ & $0.14$ & $0.68$ & $0.90$ \\
\end{tabular}
\caption{Passkey exact match in the causal one-cache task.}
\label{app:tab:causal}
\end{table}
\section{Deployment Details}\label{app:campaign}\subsection{Compared methods}

\begin{center}\footnotesize
\setlength{\tabcolsep}{5pt}
\renewcommand{\arraystretch}{1.1}
\fitwidth{\begin{tabular}{p{0.46\linewidth} | c | c | c}
& selection & runs in & frees memory \\
\hline
SnapKV, PyramidKV, TOVA, StreamingLLM, KNorm, KeyDiff, CriticalKV, Compactor, KVzap & uniform & library & yes \\
H2O & uniform & our harness & yes \\
ContourKV, budget enforced (benchmark setting) & uniform & library & yes \\
KVzip, FastKVzip, Ada-KV, ExpectedAttention, DuoAttention & per head & library & no \\
ContourKV, budget not enforced & per head & library & over-keeps \\
ContourKV, budget enforced (retrieval setting) & per head & our harness & yes \\
\end{tabular}}
\end{center}
The compared methods are SnapKV~\citep{snapkv}, PyramidKV~\citep{pyramidkv}, TOVA~\citep{tova}, StreamingLLM~\citep{streamingllm}, CriticalKV~\citep{criticalkv}, Compactor~\citep{compactor}, KVzip~\citep{kvzip}, Ada-KV~\citep{adakv}, ExpectedAttention~\citep{kvpress}, DuoAttention~\citep{duoattention}, and four methods shipped with the library, KNorm, KeyDiff, KVzap and FastKVzip. Selection is uniform when every head in a layer keeps the same number of entries, and it is per head when the counts differ. Where the library masks, our harness deletes per head, and in the benchmark setting enforcement caps each head before scoring. The unenforced row frees memory but over-keeps to each layer's widest cell, so only the two enforced rows free it fully. KVzip runs as the kvpress $0.5.4$ press with \texttt{layerwise=false}, which sets one global threshold across layers and heads, and with four sink tokens, under transformers $5.2.0$. The enforced reconstruction-scored ContourKV keeps the same score but holds exactly $s$ entries in every cell, a recency floor of $\min(32,s/2)$ positions plus the cell's own top-scoring keys. It therefore differs from KVzip in allocation (uniform against global), in its recency floor against four sinks, in physical storage against a mask, and in the question mask described below.

KNorm prunes by key norm, KeyDiff evicts the keys closest to the layer's mean key direction, and KVzap is a learned per-layer scorer shipped with the library. CompressKV is evaluated separately (App.~\ref{app:compresskv}). We configured every method to keep the same fraction of the cache, $0.1$ to $3.4\%$ across the grid, so the methods differ in which keys they keep and in how the pipeline stores them. We exclude IndexMem~\citep{indexmem}, which releases neither code nor weights, and DBTrimKV~\citep{dbtrimkv}, whose retention gates require training and whose released checkpoints do not cover any of our four models.

\paragraph{Question masks.} The pipeline builds one mask for the question forward from layer $0$'s cache length. On any compressed cache it reports a key offset of zero while the question keeps its original positions, so the causal test passes everywhere and the mask comes out all-true. Every physically evicting arm in the benchmark setting, the nine uniform library methods and both ContourKV press forms, therefore prefilled its question block bidirectionally, while the masked per-head arms kept an uncompressed contiguous cache, on which the mask stays causal. Every contrast between the two classes crosses the conventions, including $93/45/22$ against KVzip, and we have not measured the size or sign of that difference. The per-layer press form runs a strict mask that is causal among the question tokens, so its contrasts against the uniform arms cross a convention as well. On small random configurations the strict mask reproduces the sequential reference to within $10^{-4}$ in the logits, and with the same retained indices deletion preserves positions, since the cache stores post-rotary keys and the original offsets. A real-checkpoint comparison of masking and deletion on identical indices needs new runs.

\begin{center}\footnotesize
\setlength{\tabcolsep}{4pt}
\renewcommand{\arraystretch}{1.1}
\fitwidth{\begin{tabular}{l | c | c | c | c | c}
class & question seen & entries attended per head & bytes, share of dense & storage & question mask \\
\hline
evicting presses & no & $s$ & $0.09$ to $3.6\%$ (needle $0.24$ to $4.3\%$) & physical & bidirectional \\
KVzip class & no & $s$ on average & $100\%$ ($28$ to $1078\times$ SnapKV) & mask & causal \\
ContourKV, not enforced & no & up to the layer's widest cell & $3.5$ to $84\%$ ($20$ to $96\times$ SnapKV) & physical & bidirectional \\
ContourKV, enforced & no & $s$ & equal to SnapKV & physical & bidirectional \\
ContourKV, engine & in the ranking & varies by layer & $0.6$ to $1.5\times$ SnapKV (needle) & physical & causal \\
\end{tabular}}
\end{center}
\noindent Bytes held against entries attended, by method class, in the benchmark setting unless marked. The engine form evicts after the question has been prefilled over the full document, so its question block ran causally.

\subsection{Costs and uncertainty}
The byte figures of \S\ref{sec:bytes} come from per-layer memory records in our harness. Since the benchmark setting records no bytes of its own, we replayed its contexts through the same code at the same commit (prefill only, $603$ measurements per model), without touching a score. In the replay, SnapKV and Compactor each hold between $0.09$ and $3.6\%$ of the dense cache with identical byte counts, while KVzip holds all of it at every condition, which is $28$ to $1078$ times SnapKV's bytes. ContourKV without enforcement holds $3.5$ to $84\%$ of the dense cache, or $20$ to $96$ times the budget-enforcing methods' bytes and a median of $43.9$ times SnapKV's, since the unenforced mode over-keeps to each layer's widest cell (at $11$ to $85$ times its nominal budget in the needle runs).

We executed all runs on an internal SLURM cluster, allocating one NVIDIA H100 (80GB HBM3), eight CPU cores, and roughly 96 GB of host memory per job. We split the bridge capture across 168 single-GPU jobs, setting a four-hour time limit with automatic requeuing. We estimate that the core benchmark campaign took 600 GPU-hours, not counting preliminary tests and withheld variants. We recorded all H100 timings in bfloat16 as the median of five runs after two warmup iterations. For the per-layer microbenchmark in App.~\ref{app:protocol}, we also benchmarked a Xeon Gold 5418Y. SnapKV and our CompressKV reimplementation each cost $1.03$ to $1.13$ times a dense context prefill, while ContourKV's reconstruction score costs $8.7$ to $13.6$ times ($7.744$ seconds against $0.605$ on Llama-3.1-8B), although its own method class reports $2$ to $3$ times for itself.

All intervals come from one estimator, a cluster bootstrap over documents with $B=1000$ and $95\%$ percentile intervals. When both methods were measured in the same batch we take paired per-document differences, and otherwise we resample the two samples independently. Neither the resampling nor the averaging ever crosses from the benchmark setting into the retrieval one. We treat any margin below $5$ points ($0.05$ on the unit score scale) as too small to interpret, though we still count it toward the totals by its sign.

\subsection{Per-condition results}

ContourKV's margins over KVzip at RULER $s=128$ run $10.4$ to $76.4$ points across the fourteen model and length cells, and every interval excludes zero. ContourKV's only losses with intervals excluding zero, against any opponent, are LongBench repobench-p on Llama-3.1-8B at both budgets ($-7.3$ and $-6.6$ points). With the budget enforced, sixteen of the twenty-two losses to KVzip sit at RULER $s=32$, where KVzip scores $12.6$ to $17.6$ from the full cache and SnapKV scores $3.6$ to $5.5$.

\begin{table}[H]\centering\small
\setlength{\tabcolsep}{6pt}
\renewcommand{\arraystretch}{1.05}
\fitwidth{\begin{tabular}{l | c | c | c}
Opponent & Holds & Not enforced (W/T/L) & Enforced (W/T/L) \\
\hline
KVzip$^\ddagger$ & full cache & $\mathbf{146/9/2}$ & $\mathbf{93/45/22}$ \\
ExpectedAttention$^\ddagger$ & full cache & $142/15/0$ & $61/80/19$ \\
FastKVzip$^\ddagger$ & full cache & $74/3/0$ & $57/23/0$ \\
SnapKV & budget & $137/19/1$ & $56/91/13$ \\
Compactor & budget & $136/20/1$ & $\mathbf{24/96/40}$ \\
KVzap & budget & $76/2/1$ & $38/28/14$ \\
\end{tabular}}
\caption{Benchmark record: win/tie/loss for ContourKV against each opponent over paired conditions, without and with its budget enforced. $^\ddagger$These arms kept a causal question mask while ContourKV's question block ran bidirectionally, so their contrasts cross mask conventions. Against KVzip the enforced record is $+6.7$ and $+7.5$ points on LongBench at $s=32$ and $128$, $-0.06$ on RULER $s=128$ and $-9.41$ ($0/0/16$) on RULER $s=32$. Beyond the $5$-point noise floor the enforced record is $54/19$ against KVzip and $13/13$ against Compactor. Capping the unenforced selection at the nominal budget scores $0/17/140$ against it, by $-14.1$, $-19.5$, $-35.8$ and $-62.2$ points on LongBench $s=32$, LongBench $s=128$, RULER $s=32$ and RULER $s=128$.}
\label{app:tab:wtl}
\end{table}

\begin{table}[H]\centering\small
\renewcommand{\arraystretch}{1.05}
\fitwidth{\begin{tabular}{l | c | l | c}
Opponent & W/T/L & Opponent & W/T/L \\
\hline
SnapKV & $\mathbf{23/23/2}$ & KVzip & $27/21/0$ \\
PyramidKV & $23/23/2$ & Ada-KV & $27/17/4$ \\
TOVA & $25/23/0$ & ExpectedAttention & $27/21/0$ \\
StreamingLLM & $23/24/1$ & H2O & $12/6/0$ ($18$) \\
KNorm & $27/21/0$ & KVzap & $11/8/0$ ($19$) \\
KeyDiff & $21/26/1$ & FastKVzip & $8/16/2$ ($26$) \\
CriticalKV & $23/24/1$ & DuoAttention & $9/2/2$ ($13$) \\
Compactor & $24/22/2$ & & \\
\end{tabular}}
\caption{Retrieval at the joint-prefill condition of \S\ref{sec:needle}: win/tie/loss for the window-scored ContourKV engine form, whose ranking runs over question rows, against each compared method over document-paired conditions, with $48$ per opponent unless noted in parentheses.}
\label{app:tab:needle}
\end{table}

\begin{table}[H]\centering\small
\setlength{\tabcolsep}{6pt}
\renewcommand{\arraystretch}{1.1}
\fitwidth{\begin{tabular}{l | c | c}
Re-run & Change in hit rate & Memory \\
\hline
ranking restricted to document rows & $\mathbf{-88.2}$ $[-90.0,-86.4]$ & $1.02\times$ \\
eviction moved to the document boundary & $\mathbf{-87.6}$ $[-89.0,-86.3]$ & $1.43\times$ \\
difference, the question's attention to the document & $+0.6$ & \\
\end{tabular}}
\caption{The condition decomposition of \S\ref{sec:needle}, on eight document-paired Llama-3.1-8B cells at $s\in\{64,128\}$. Changes are hit-rate points against the unmodified run, and memory is relative to the run compressed before the question. The re-runs are nested on the same cells, since the boundary move removes the question rows from the ranking and also the question's attention to the uncompressed document, so their difference prices the second.}
\label{app:tab:condition}
\end{table}

Although it holds $1.43$ times the memory, the re-run that evicts at the document boundary finishes only $2.6$ points above the run compressed before the question, so the fall is not a memory effect. The re-run restricted to document rows does evict four to six of the question's thirty-nine rows, but that is far too few to account for an $88$-point drop.

\paragraph{Equal-memory allocation and retrieval.} Every needle tally is at the joint-prefill condition of \S\ref{sec:needle}, where the engine form ranks over question rows and the reading-time presses do not. Over the $48$ conditions the tally against SnapKV is $23/23/2$, and the ordering holds against the other fourteen methods at comparable bytes (Table~\ref{app:tab:needle}). At $s=32$ neither method recovers the planted fact. At RULER $s=128$ the window-scored per-layer variant falls $13.55$ points below Compactor and $10.75$ points below the uniform run. In the retrieval runs, the $s=64$ record against SnapKV is $7/7/1$, every advantage in the $48$-condition tally falls at $s\ge64$, and the window-scored version holds $0.82$ times SnapKV's bytes at $s=64$ and $1.17$ times at $s=128$. The reconstruction score concentrates the budget so strongly that padding each layer's block to its widest head consumes up to seven eighths of the memory at grouped-query width $8$, and we stopped that variant when its answer-span coverage fell $5.4$ to $31.6$ points below the uniform run, as the rule that we had fixed in advance required. Qwen3-14B fails retrieval under both the window-scored version and SnapKV for reasons specific to that lineage.

\paragraph{CompressKV.}\label{app:compresskv} CompressKV~\citep{compresskv2} is the one extreme-compression challenger with released code, and we evaluate it in a query-agnostic setting that its published scores do not measure. Our reimplementation selects the same entries as the official implementation on all $328$ test cases, and the one unavoidable question-time code-path difference moves bfloat16 scores by at most $1.22\times10^{-4}$. Its complete record covers $80$ cells at exactly matched memory on Llama-3.1-8B and Qwen2.5-7B, the two models for which its authors provide calibration. It is $7/47/26$ against ContourKV with the budget enforced ($-0.85$ points in aggregate), with the RULER $s=128$ block at $0/0/8$ and $-9.39$. It is $13/60/7$ against SnapKV at $+0.89$, $15/55/10$ against the per-layer variant of \S\ref{sec:square} at $+0.79$ $[+0.24,+1.38]$, which is its one favorable comparison and is concentrated on Llama-3.1-8B, and $11/36/33$ against Compactor, where it trails by $4.96$ at RULER $s=32$ (ContourKV's own aggregate against Compactor runs $-1.9$ to $+0.7$). Its layer allocation does not vary at $s=32$ (a uniform $32$ entries per layer) and varies at $s=128$ ($54$ to $169$ per layer on Llama-3.1-8B and $63$ to $174$ on Qwen2.5-7B), so the $s=32$ comparison tests its scoring alone and $s=128$ adds its allocation, where the deficit is largest.
\section{Attention Error and Cross-Entropy Loss} \label{app:scope}
The prefill probe of \S\ref{sec:teeth} is a rank-$16$ scorer trained on the query and key projections of Qwen2.5-1.5B and Llama-3.2-1B to reproduce the balancing swaps~\citep{oracleprefill}, and it recovers at most $5\%$ of the decode-time loss reduction against a $25\%$ bar fixed in advance. It is zero-initialized, so an untrained probe selects identically to top-mass, and it was trained on $48$ documents disjoint from evaluation, where the mixing weight between the probe and the top-mass score was calibrated on half of the held-out documents and checked on the other half. Its held-out swap recall is $0.33$ and $0.39$, and when deployed it moves held-out cross-entropy at $s=8$ by $+0.010$ and $-0.006$, where balancing in the same run moves it by $-0.119$ and $-0.118$.

To test whether the map picks the right cells, we switched the same number of cells at random instead. On Qwen2.5-1.5B the map recovers $94\%$ of the all-cell gain and the random set recovers $66\%$, so two thirds of the map's gain comes from set size alone; on Qwen3-1.7B the random set recovers $77\%$ and the map recovers $15\%$, and no ranking of cells that we tried beat the random set (Fig.~\ref{fig:control}). Across the three Qwen3 scales the map keeps $55$/$15$/$31\%$ of the all-cell gain. Both configurations sit far from the dense output, so the difference between their cross-entropies is second order and is set by the local curvature. At deployment scale, the reconstruction-ranked per-layer batch was withheld when its answer-span coverage collapsed at matched memory (\S\ref{sec:square}), and the reconstruction-scored state of the art loses when scored at its measured bytes (\S\ref{sec:bytes}).

\begin{figure}[H]\centering
\includegraphics[width=0.95\linewidth]{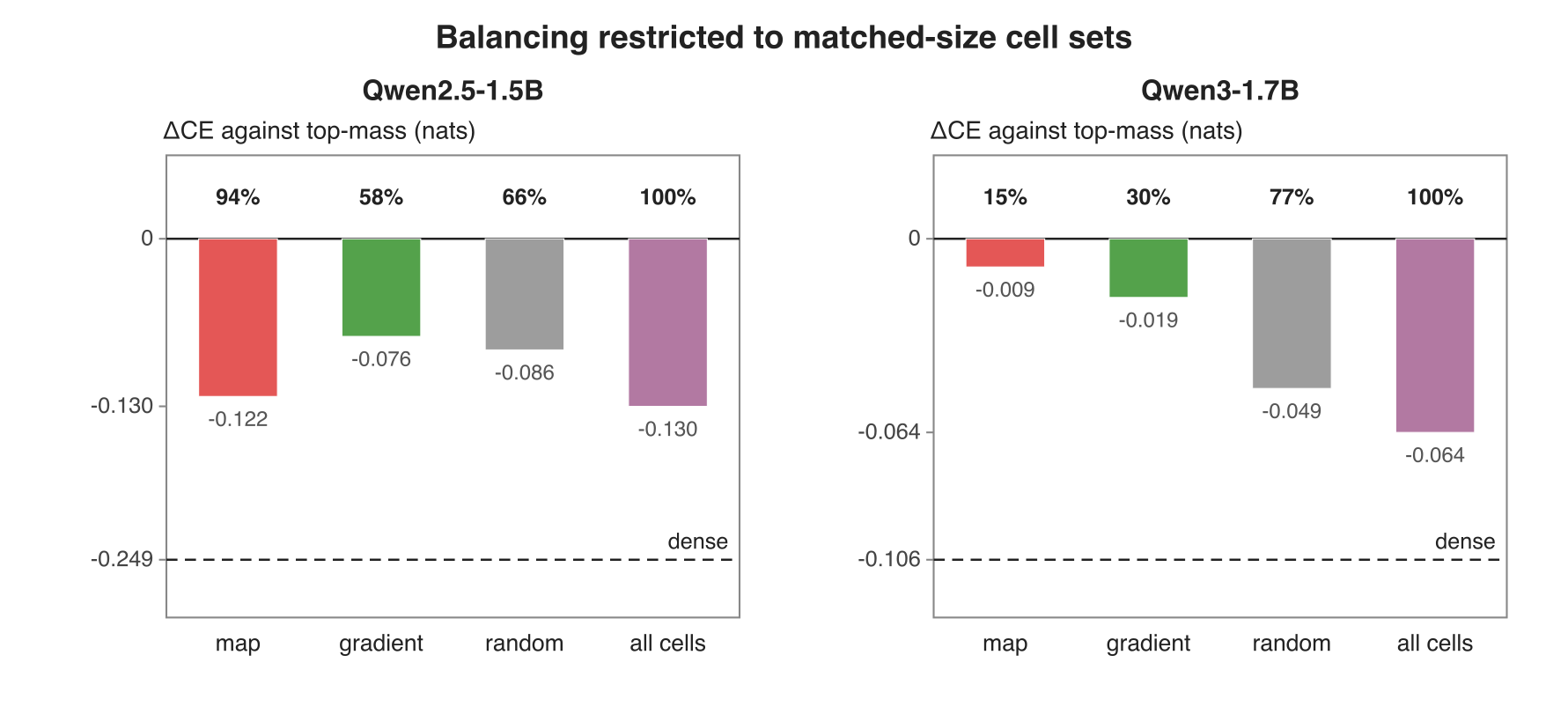}
\caption{Change in held-out cross-entropy against top-mass at $s=8$ when balancing only the labeled cells. The percentages give each set's fraction of the gain from balancing all cells.}
\label{fig:control}
\end{figure}

We measured the local curvature directly, from loss gradients and Hessian-vector products at each head's output-projection input. The second-order prediction built from it matches the measured shares of the loss gain ($0.86$ to $0.95$ against $0.88$ to $0.98$), and the conversion constant is the loss change per unit of flat-cell reconstruction reduction.

\section{Related Work}\label{app:related}
Forced weights separate this problem from prior work. Approximating a mean by a few of its points is classical, and every near-optimal answer re-solves the weights on the points that it keeps. Renormalization removes that freedom, since it only rescales the original weights. Resampling keeps the largest weights, which is optimal only for value-agnostic unbiased objectives~\citep{varopt,fearnhead}. Optimal design relaxes the subset to a re-solved design measure~\citep{design}. Coresets and recombination reweight whatever they retain~\citep{feldman-langberg,recombination,caratheodory-approx}. The attention-discrepancy line applies vector balancing, with re-solved weights in~\cite{balancekv} and one uniform rescaling of the kept weights in~\cite{attncoresets}. Sampling estimators reweight the keys that they draw with the decode query over a full cache~\citep{magicpig,vattention}, so their weights are not forced either. None of them answers which subset is best when the weights are fixed.

In the eviction literature the objects differ in another way. Value-aware scores rank one decision at decode time and do not certify the set that remains~\citep{caote,vatp,criticalkv}, and leverage selection builds sets that already contain the high-mass keys~\citep{levattention}. Retained mass is deployed practice, used to stop a budget~\citep{twilight} and to allocate one across heads~\citep{adakv}, with no statement of how far the resulting set sits from the optimum. We supply that statement. The nearest concurrent diagnostic works per block, so it certifies nothing per head~\citep{zhang2026doesvalueawarekveviction}. The single-key case of Lemma~\ref{lem:identities} recovers the per-decision criterion of these value-aware scores (App.~\ref{app:proofs}), so the balancing selector greedily minimizes, over a whole kept set, the output error that they rank one decision at a time. CriticalKV is in our deployed roster. Their gain sits in the value-side residual $g_A$, of which a probe on the query and key projections recovers at most $5\%$ (App.~\ref{app:scope}), and that gain can be larger in loss than in reconstruction, since the balancing oracle closes about half of the cross-entropy gap at $s=8$ (Table~\ref{app:tab:celoss}).

On the systems side, ragged paging is being built for the storage constraint that we measure~\citep{tangram}, and the nominal-to-physical memory gap is posed as an infrastructure problem~\citep{nvidiainfra}, which we measure per method class. We audit the deployed compressors, in particular the reconstruction-scored class~\citep{kvzip}, the observation-window class~\citep{snapkv}, and the strongest budget-enforcing baseline~\citep{compactor}. ContourKV reuses the importance scores of the first two classes and replaces only how a kept set is chosen from them and stored (\S\ref{sec:ContourKV}).

\section{Limitations}\label{app:limits}
Our bound is one-sided, since $\hat c\le c$ on every instance while the upper side rests on search. The full-key solver settled $159$ of its $939$ solved instances, and the caps of App.~\ref{app:cert} hold only if its open brackets hide no further improvement, so a certified upper bound on the population median of $c$ needs new solves.

Reconstruction is not loss, and selection still matters for loss. The balancing selector, which reads $\mu$, closes $0.43$ to $0.60$ of the top-mass-to-dense cross-entropy gap at $s=8$ on seven families (Table~\ref{app:tab:celoss}), while a probe on the query and key projections recovers at most $5\%$ of that gain on the two families tested (App.~\ref{app:scope}). We do not optimize loss, KL divergence or generation quality directly.

Keeping each row's largest weights is a per-row target. One kept set shared by a query group already costs each row a median $1.14$ to $1.38$ times its own $\ES_{\mathrm{pool}}$ (App.~\ref{app:perheadtables}), and as an eviction rule for one cache that serves later queries, top-mass eviction trails a running SnapKV by $0.53$ to $0.89$ in passkey exact match at $s\ge128$ on Llama-3.2-1B and Qwen2.5-1.5B, though not on three Qwen3 models (Table~\ref{app:tab:causal}).

The measurement panel covers dense attention at $0.6$B to $14$B and up to $16{,}384$ tokens, with no mixture-of-experts or latent-attention model. The deployment runs two benchmarks through one pinned pipeline (kvpress $0.5.4$, transformers $5.2.0$), whose question masks differ between the physically evicting and the masked arms (App.~\ref{app:campaign}), and the retrieval tallies hold at the joint-prefill condition of \S\ref{sec:needle}. A matched-memory, matched-condition allocation gain and a real-checkpoint comparison of masking and deletion on identical indices need new runs.


\end{document}